\documentclass[11pt]{article}
\usepackage{times}
\usepackage[letterpaper, left=1in, right=1in, top=1in, bottom=1in]{geometry}

\usepackage{parskip}

\usepackage[dvipsnames]{xcolor}
\usepackage[colorlinks=true, linkcolor=blue, citecolor=blue, urlcolor=black]{hyperref}
\usepackage{microtype}

\usepackage{algorithm}
\usepackage{algpseudocode}

\usepackage{natbib}
\bibpunct{(}{)}{;}{a}{,}{,}

\usepackage{amsthm}
\usepackage{mathtools}
\usepackage{amsmath}
\usepackage{bbm}
\usepackage{amsfonts}
\usepackage{amssymb}

\usepackage{MnSymbol} %Actually conflicts with amssymb and others

\usepackage{xpatch}

\newtheorem{theorem}{Theorem}

\newtheorem{definition}{Definition}
\theoremstyle{definition}  %Sets style of subsequent newtheorems to 'definition'

\newtheorem{lemma}{Lemma}
\newtheorem{conjecture}{Conjecture}
\newtheorem{corollary}{Corollary}
\newtheorem{proposition}{Proposition}

\xpatchcmd{\proof}{\itshape}{\normalfont\proofnameformat}{}{}
\newcommand{\proofnameformat}{\bfseries}

\usepackage{prettyref}
\newcommand{\pref}[1]{\prettyref{#1}}
\newcommand{\pfref}[1]{Proof of \prettyref{#1}}
\newcommand{\savehyperref}[2]{\texorpdfstring{\hyperref[#1]{#2}}{#2}}
\newrefformat{eq}{\savehyperref{#1}{\textup{(\ref*{#1})}}}
\newrefformat{eqn}{\savehyperref{#1}{Equation~\ref*{#1}}}
\newrefformat{lem}{\savehyperref{#1}{Lemma~\ref*{#1}}}
\newrefformat{def}{\savehyperref{#1}{Definition~\ref*{#1}}}
\newrefformat{fact}{\savehyperref{#1}{Fact~\ref*{#1}}}
\newrefformat{line}{\savehyperref{#1}{line~\ref*{#1}}}
\newrefformat{thm}{\savehyperref{#1}{Theorem~\ref*{#1}}}
\newrefformat{cor}{\savehyperref{#1}{Corollary~\ref*{#1}}}
\newrefformat{sec}{\savehyperref{#1}{Section~\ref*{#1}}}
\newrefformat{app}{\savehyperref{#1}{Appendix~\ref*{#1}}}
\newrefformat{ass}{\savehyperref{#1}{Assumption~\ref*{#1}}}
\newrefformat{ex}{\savehyperref{#1}{Example~\ref*{#1}}}
\newrefformat{fig}{\savehyperref{#1}{Figure~\ref*{#1}}}
\newrefformat{alg}{\savehyperref{#1}{Algorithm~\ref*{#1}}}
\newrefformat{rem}{\savehyperref{#1}{Remark~\ref*{#1}}}
\newrefformat{conj}{\savehyperref{#1}{Conjecture~\ref*{#1}}}
\newrefformat{prop}{\savehyperref{#1}{Proposition~\ref*{#1}}}
\newrefformat{proto}{\savehyperref{#1}{Protocol~\ref*{#1}}}
\newrefformat{prob}{\savehyperref{#1}{Problem~\ref*{#1}}}
\newrefformat{claim}{\savehyperref{#1}{Claim~\ref*{#1}}}

\DeclarePairedDelimiter{\abs}{\lvert}{\rvert} %
\DeclarePairedDelimiter{\brk}{[}{]}
\DeclarePairedDelimiter{\crl}{\{}{\}}
\DeclarePairedDelimiter{\prn}{(}{)}
\DeclarePairedDelimiter{\nrm}{\|}{\|}
\DeclarePairedDelimiter{\tri}{\langle}{\rangle}

\DeclarePairedDelimiter{\ceil}{\lceil}{\rceil}

\let\Pr\undefined

\DeclareMathOperator{\En}{\mathbb{E}}
\newcommand{\Ep}{\wt{\bbE}}
\DeclareMathOperator*{\Eh}{\widehat{\mathbb{E}}}

\DeclareMathOperator{\Pr}{Pr}

\DeclareMathOperator*{\argmin}{arg\,min} % * Places subscript directly under operator
\DeclareMathOperator*{\argmax}{arg\,max}

\newcommand{\ls}{\ell}
\newcommand{\ind}{\mathbbm{1}}    %Indicator
\newcommand{\pmo}{\crl*{\pm{}1}}
\newcommand{\eps}{\epsilon}
\newcommand{\veps}{\varepsilon}

\newcommand{\ldef}{\vcentcolon=}

\newcommand{\wt}[1]{\widetilde{#1}}
\newcommand{\wh}[1]{\widehat{#1}}

\def\ddefloop#1{\ifx\ddefloop#1\else\ddef{#1}\expandafter\ddefloop\fi}
\def\ddef#1{\expandafter\def\csname bb#1\endcsname{\ensuremath{\mathbb{#1}}}}
\ddefloop ABCDEFGHIJKLMNOPQRSTUVWXYZ\ddefloop
\def\ddefloop#1{\ifx\ddefloop#1\else\ddef{#1}\expandafter\ddefloop\fi}
\def\ddef#1{\expandafter\def\csname b#1\endcsname{\ensuremath{\mathbf{#1}}}}
\ddefloop ABCDEFGHIJKLMNOPQRSTUVWXYZ\ddefloop
\def\ddef#1{\expandafter\def\csname c#1\endcsname{\ensuremath{\mathcal{#1}}}}
\ddefloop ABCDEFGHIJKLMNOPQRSTUVWXYZ\ddefloop
\def\ddef#1{\expandafter\def\csname h#1\endcsname{\ensuremath{\widehat{#1}}}}
\ddefloop ABCDEFGHIJKLMNOPQRSTUVWXYZ\ddefloop
\def\ddef#1{\expandafter\def\csname hc#1\endcsname{\ensuremath{\widehat{\mathcal{#1}}}}}
\ddefloop ABCDEFGHIJKLMNOPQRSTUVWXYZ\ddefloop
\def\ddef#1{\expandafter\def\csname t#1\endcsname{\ensuremath{\widetilde{#1}}}}
\ddefloop ABCDEFGHIJKLMNOPQRSTUVWXYZ\ddefloop
\def\ddef#1{\expandafter\def\csname tc#1\endcsname{\ensuremath{\widetilde{\mathcal{#1}}}}}
\ddefloop ABCDEFGHIJKLMNOPQRSTUVWXYZ\ddefloop

\newcommand{\Holder}{H{\"o}lder}
\let\wt\undefined

\newcommand{\wt}[1]{\widetilde{#1}}

\newcommand{\W}{\mathcal{W}}

\newcommand{\grad}{\nabla}

\newcommand{\tens}{\otimes{}}
\newcommand{\trn}{\intercal}

\newcommand{\rank}{\mathrm{rank}}

\renewcommand{\trn}{\dagger}

\usepackage[mathscr]{euscript}

\usepackage{xspace}

\newcommand{\inj}{\mathsf{inj}}
\newcommand{\op}{\mathsf{op}}
\newcommand{\nuc}{\mathsf{nuc}}
\newcommand{\injt}{\wt{\mathsf{inj}}}
\newcommand{\nuct}{\wt{\mathsf{nuc}}}

\newcommand{\vecone}{x}
\newcommand{\vectwo}{y}
\newcommand{\vecthree}{z}

\newcommand{\QTperp}{\cQ_{T^{\perp}}}
\newcommand{\QTpara}{\cQ_{T^{\para}}}

\newcommand{\Ot}{\wt{O}}

\newcommand{\approxleq}{\lesssim}

\newcommand{\entails}{\vDash}
\newcommand{\provable}{\vdash}

\newcommand{\midsem}{\mathbin{;}}

\renewcommand{\Pr}{\bbP}

\newcommand{\para}{\parallel}

\renewcommand{\dim}{\mathrm{dim}\,}
\renewcommand{\trn}{\top}

\newcommand{\flatten}{\mathord{\flat}}

\newcommand{\RTthree}{(\bbR^{d})^{\tens{}3}}

\newcommand{\dataop}{\mathscr{X}_n}

\renewcommand{\wr}{\mathrm{wr}}
\newcommand{\wo}{\mathrm{w/o}}

\newcommand{\dwr}{\cD^{n}_{\wr}}
\newcommand{\dwo}{\cD^{n}_{\wo}}

\newcommand{\sosabbrev}{SoS\xspace}
\newcommand{\sos}{\sosabbrev}
\newcommand{\soslong}{sum-of-squares\xspace}

\newcommand{\specnorm}{\mathsf{op}}
\newcommand{\frobnorm}{F}

\newcommand{\poprisk}{L_{\cD}}
\newcommand{\emprisk}{\wh{L}_{n}}
\newcommand{\term}{\wh{T}_{n}}
\newcommand{\ts}{T^{\star}}
\renewcommand{\Eh}{\wh{\En}_n}

\newcommand{\propone}{\savehyperref{thm:generic_offset}{Property 1}\xspace}
\newcommand{\proptwo}{\savehyperref{thm:generic_offset}{Property 2}\xspace}
\newcommand{\propthree}{\savehyperref{thm:generic_offset}{Property 3}\xspace}

\newcommand{\xor}{3-XOR\xspace}

\renewcommand{\paragraph}[1]{\par\textbf{#1}\hspace{5pt}}

\usepackage[suppress]{color-edits}
\addauthor{df}{red}
\addauthor{ar}{blue}

\title{\huge Sum-of-squares meets square loss: \\ Fast rates for agnostic tensor completion}

\author{
\begin{tabular}{c}
{\Large Dylan J. Foster~~~~~~~~~~Andrej Risteski}\\
Massachusetts Institute of Technology\\
{\small\texttt{\{dylanf,risteski\}@mit.edu}}
\end{tabular}
}
\date{}

\begin{document}

\maketitle

\begin{abstract}%
We study tensor completion in the \emph{agnostic setting}. In the
classical tensor completion problem, we receive $n$ entries of an
unknown rank-$r$ tensor and wish to exactly complete the remaining entries. In agnostic tensor completion, we make \emph{no assumption} on the
rank of the unknown tensor, but attempt to \emph{predict} unknown
entries as well as the best rank-$r$ tensor.

For agnostic learning of third-order tensors with the square loss, we
give the first polynomial time algorithm that obtains a ``fast''
(i.e., $O(1/n)$-type) rate improving over the rate obtained by
reduction to matrix completion. Our prediction error rate to compete
with the best $d\times{}d\times{}d$ tensor of rank-$r$ is
$\wt{O}(r^{2}d^{3/2}/n)$. We also obtain an exact oracle inequality that trades off estimation and approximation error.

Our algorithm is based on the degree-six sum-of-squares relaxation of the tensor nuclear norm. The
key feature of our analysis is to show that a certain characterization for the
subgradient of the tensor nuclear norm can be encoded in the
sum-of-squares proof system. This unlocks the standard toolbox for
localization of empirical processes under the square loss, and allows
us to establish restricted eigenvalue-type guarantees for various tensor
regression models, with tensor completion as a special case. The new analysis of the relaxation complements
\cite{barak2016noisy}, who gave slow rates for agnostic tensor
completion, and \cite{potechin2017exact}, who gave exact recovery
guarantees for the noiseless setting. Our techniques are
user-friendly, and we anticipate that they will find use elsewhere. 

\end{abstract}

%\begin{keywords}%
%tensor completion, sum-of-squares, statistical learning, agnostic learning, fast rates, oracle inequalities,  localization
%\end{keywords}

\section{Introduction}
\label{sec:introduction}
% !TEX root = paper.tex

Recovering structured mathematical objects from partial measurements is a fundamental task in machine learning and statistical inference. One important example, which has been a mainstay of modern research in machine learning and high-dimensional statistics, is \emph{matrix completion}. Here, we receive $n$ entries from an unknown $d\times{}d$ matrix, and the goal is to complete the remaining entries when $n$ is as small is possible. The key structural assumption that enables recovery when $n\ll{}d^{2}$ is that the underlying matrix is \emph{low-rank}. A celebrated line of research on matrix completion \citep{srebro2005rank,candes2009exact,candes2010power,keshavan2010matrix,gross2011recovering,recht2011simpler} has culminated in the following guarantee: to \emph{exactly} recover an incoherent rank-$r$ matrix, $n=\wt{O}(rd)$ uniformly sampled entries suffice.

While low-rank matrix completion has seen successful application across many problem domains (most famously in the context of the \emph{Netflix Problem}), for many tasks it is natural to consider not just pairwise interactions but higher-order interactions, leading to the problem of \emph{tensor completion}. Tensor completion poses significant computational hurdles compared to the matrix case, but in an impressive recent work, \cite{potechin2017exact} provide an efficient algorithm based on the sum-of-squares hierarchy that exactly recovers a $d\times{}d\times{}d$ tensor of rank-$r$ with incoherent and orthogonal components from $\wt{O}(rd^{3/2})$ measurements. While this undershoots the optimal statistical rate of $\Ot(rd)$, there is evidence that this is optimal amongst polynomial time algorithms under certain average-case hardness assumptions \citep{barak2016noisy}.

In real-world applications, the assumption that the underlying tensor is truly low-rank may be too strong, and model misspecification is unavoidable. This is the main thrust of \emph{agnostic learning} \citep{haussler1992decision,kearns1994toward} which, rather than attempting to recover an unknown model in some class, attempts to \emph{predict} as well as the best model. The aim of this paper is to develop guarantees for \emph{agnostic tensor completion}: predicting as well as the best low rank tensor, even when exact recovery is impossible. 

We work in the following \emph{agnostic tensor regression} model, which captures tensor completion as a special case: we receive examples $(X_1,Y_1,),\ldots,(X_n,Y_n)$ i.i.d. from an unknown distribution $\cD$, where each instance $X_t$ is a $d\times{}d\times{}d$ tensor and $Y_t$ is a real-valued response. Letting $\tri*{\cdot,\cdot}$ denote the usual inner product, we measure predictive performance of a given $d\times{}d\times{}d$ tensor $T$ via its \emph{square loss risk} $\poprisk(T)=\En_{(X,Y)\sim{}\cD}\prn*{\tri{T,X}-Y}^{2}$. Our goal is to use the samples to produce a predictor $\term$ that enjoy low \emph{excess risk}:
\begin{equation}
\label{eq:excess_risk}
\poprisk(\term) - \inf_{T : \text{ rank-$r$}}\poprisk(T) \leq{} \veps(n,r,d),
\end{equation}
where the bound $\veps(n,r,d)$ converges to zero as $n\to\infty$. When the observations $X$ are uniformly distributed indicators (i.e. $X=e_{i}\tens{}e_j\tens{}e_k$, where $(i,j,k)$ is uniform) this recovers the usual measurement model for tensor completion. If the model is \emph{well-specified} in the sense that $Y=\tri*{X,T^{\star}}+\xi$, where $T^{\star}$ is a low-rank tensor and $\En\brk*{\xi\mid{}X}=0$, then low excess risk implies approximate recovery of $T^{\star}$. In general, the guarantee \pref{eq:excess_risk} is interesting because it implies non-trivial predictive performance even in the presence of severe model misspecification.\vspace{5pt}

In the matrix case, agnostic excess risk guarantees of the type in \pref{eq:excess_risk} were characterized by \cite{koltchinskii2011nuclear}. There it was shown that to compete with the best rank-$r$ matrix with bounded entries, empirical risk minimization with nuclear norm penalization obtains a \emph{fast rate} of the form $\veps(n,r,d) = O\prn*{rd\log{}d/n}$, and also showed that this is optimal.\footnote{Interestingly, their results also show that incoherence---usually taken as necessary in matrix completion for positive results ---is not necessary to obtain prediction error bounds; boundedness suffices.} The nomenclature fast rate is intended to contrast the $1/n$ dependency on $n$ with \emph{slow rates}, which have a $1/\sqrt{n}$ dependency on $n$, and which are typically much simpler to obtain (see \cite{gaiffas2011sharp} for slow rates in the matrix setting). \vspace{3pt}

The results of \cite{koltchinskii2011nuclear} leverage strong understanding of the (matrix) nuclear norm, namely matrix concentration and decomposability/subgradient properties. In this paper, we tackle the following questions:
\begin{itemize} 
\item Can we give similar guarantees for agnostic \emph{tensor} completion? 
\item Can we obtain fast rates while at the same time relaxing the strong statistical assumptions (low rank observations, incoherence) needed for exact recovery?
\item What are the best rates we can obtain subject to employing a polynomial-time algorithm? 
\end{itemize}

\subsection{Our contributions}
Our main result is to give the first polynomial time algorithm with a fast $O(1/n)$-type rate for agnostic completion of third order tensors that improves over the rate obtained by the natural reduction to matrix completion. Our main theorem gives excess risk bounds relative to low-rank \emph{orthogonal} tensors, i.e. tensors of the form
\begin{equation}
  \label{eq:orthogonal}
  T = \sum_{i=1}^{r}\lambda_i\cdot{}u_i\tens{}v_i\tens{}w_i,
\end{equation}
where $\nrm*{u_i}_2=\nrm*{v_i}_2=\nrm*{w_i}_2=1$ and
$\crl*{u_i}$ are orthogonal, as are $\crl*{v_i}$ and $\crl*{w_i}$. The result is as follows.%\footnote{See \pref{sec:main_results} for the full version.}
\begin{theorem}[informal]
  \label{thm:tensor_completion_informal}
There is a convex set of $d\times{}d\times{}d$ tensors $\cT$ derived from a sum-of-squares relaxation of the tensor nuclear norm, for which the empirical risk minimizer $\wh{T}_n\ldef\argmin_{T\in\cT}\frac{1}{n}\sum_{t=1}^{n}\prn*{\tri*{T,X_t}-Y_t}^{2}$ can be computed in polynomial time, and guarantees that with probability at least $1-o(1)$,
\begin{equation}
  \label{eq:tensor_completion_informal}
\poprisk(\term) - \inf_{T:\textnormal{ rank-$r$, orthogonal}}\poprisk(T) \leq{}
\wt{O}\prn*{\frac{r^{2}d^{3/2}}{n}}.
\end{equation}
The guarantee applies to both random indicator measurements (tensor completion) and gaussian measurements (tensor compressed sensing).  
\end{theorem}
The full version of the theorem is stated in \pref{sec:main_results}. The result be thought of as a generalization of the agnostic matrix completion results of \cite{koltchinskii2011nuclear} to higher-order tensors. 
We also achieve a more general \emph{exact oracle inequality} that trades off approximation and estimation error. This takes the form%, i.e., a guarantee of the form
\[
\poprisk(\term) \leq{}  \inf_{T:\textnormal{ orthogonal}}\crl*{\poprisk(T) + 
\wt{O}\prn*{\frac{r^{2}(T)d^{3/2}}{n}}},
\]
where $r(T)$ denotes the rank of $T$.\footnote{Here the word ``exact'' refers to the fact that the leading constant in front of the loss on the right-hand-side is $1$; such guarantees generally do not easily follow from the usual machinery used to analyze well-specified models.}

\paragraph{Algorithm.} Our algorithm is based on the sum-of-squares (\sos) hierarchy of convex relaxations \citep{shor1987approach,parrilo2000structured,lasserre2001global}, applied to the tensor nuclear norm. For $k\in\bbN$, the \soslong hierarchy defines a sequence of outer convex relaxations $\nrm*{\cdot}_{\nuct_k}$ that give an increasingly tight approximation to the (hard to approximate \citep{hillar2013most}) tensor nuclear norm $\nrm*{\cdot}_{\nuc}$ as $k$ is increased. The \sos nuclear norm was previously used in the work of \cite{barak2016noisy}, who gave $O(1/\sqrt{n})$ rates for tensor completion in the agnostic model with absolute loss. Their main contribution was to show how certain spectral bounds arising in 3-SAT refutation \citep{coja2004strong} bound the Rademacher complexity for the unit ball of the degree-six \sos nuclear norm, thereby controlling the usual empirical process via uniform convergence. Our guarantees are based on empirical risk minimization over the (scaled) ball in this norm, specifically the degree-six relaxation $\nrm*{\cdot}_{\nuct_6}$, and our analysis builds on their Rademacher complexity bounds.

\paragraph{Restricted eigenvalue and subgradient lemmas.}
While the $O(1/\sqrt{n})$-type rates provided in \cite{barak2016noisy} are optimal (in terms of $n$ dependence) for generic Lipschitz losses, it is not immediately obvious whether their result can be used to provide $O(1/n)$-type fast rates for strongly convex losses like the square loss. It has long been recognized that to obtain fast rates for prediction with strongly convex losses, more refined control of the empirical process is necessary. In particular, it is well-known that \emph{local} Rademacher complexities and related fixed point complexities characterize the rates for empirical risk minimization in a data-dependent fashion. To bound such localized complexities under convex relaxations, the typical approach is to establish a restricted eigenvalue property for the empirical design matrix \citep{negahban2012unified,bartlett2012l1,lecue2017regularization,lecue2018regularization}. Establishing such guarantees for regularizers such as the $\ls_1$ norm, nuclear norm, and so on is usually done by appealing to properties of the subgradient of the norm and proving that the norm (approximately) decomposes across certain subspaces \citep{negahban2012unified}. The main tool we establish here, which allows us to unlock the full power of localized complexities and establish rates fast rates, is a new guarantee of this type for the \sos nuclear norm. Informally, we show:
\begin{theorem}[informal]
\label{thm:norm_comparison_informal}
Let $T^{\star}$ be a fixed orthogonal tensor, and let $\cT = \{T: \nrm*{T}_{\nuct_6} \leq \nrm*{T^\star}_{\nuct_6} \}$. Then:  
%be the radius-$\nrm*{T^\star}_{\nuct_6}$ ball in the degree-six \sos nuclear norm. Then we have
\[
\nrm*{T-T^{\star}}_{\nuct_4}\leq{} O(r(T^{\star}))\cdot\nrm*{T-T^{\star}}_{\frobnorm}\quad\forall{}T\in\cT.
\]
\end{theorem}
This theorem is a consequence of a more general result we prove in \pref{sec:subgradient}, which gives a characterization of the subgradient of $\nrm*{\cdot}_{\nuct_4}$ at any orthogonal tensor. The basic idea is to show that a characterization for the subgradient of the tensor nuclear norm given in \cite{yuan2016tensor} can be captured in the sum-of-squares proof system. %Our proof borrows some ideas that were used in \cite{potechin2017exact} to construct dual certificates for exact tensor completion. 
The final result (\pref{thm:sos_subgradient}) is fairly user-friendly, and packages the complexity of \sos into a self-contained statement about geometric properties of the norm $\nrm*{\cdot}_{\nuct_4}$. We hope that this will find use in other applications. As one concrete example, in \pref{sec:main_results} we show how the subgradient characterization can also be used to give agnostic fast rates for the problem of \emph{low-rank tensor sensing} with gaussian measurements. Here we also obtain $\wt{O}(r^{2}d^{3/2}/n)$-type fast rates.

\paragraph{Lower bounds.} Lastly, we prove that our results can't be strengthened significantly while still obtaining computationally efficient algorithms. It is straightforward to show that there is an \emph{inefficient} predictor that obtains $O(rd/n)$ excess risk, whereas the dimension scaling in \pref{thm:tensor_completion_informal} is $O(d^{3/2})$. This scaling is shared by the other results based on the \sos nuclear norm \citep{barak2016noisy,potechin2017exact}, and \cite{barak2016noisy} show that finding relaxations for which the Rademacher complexity grows as $o(d^{3/2})$ is at least as hard as refuting random instances of \xor. In \pref{sec:lower_bounds} we give a computational lower bound for agnostic learning that shows that obtaining square loss excess risk scaling as $o(d^{3/2})$ is at least as hard as a certain distinguishing problem for random \xor, sometimes called learning sparse parities with noise.

\subsection{Related work}
Early algorithms for computationally efficient tensor completion relied on \emph{unfolding}: reshaping the tensor into a matrix and applying a matrix completion algorithm \citep{tomioka2011statistical,tomioka2013convex,romera2013new,mu2014square,jain2014provable}.
This approach yields suboptimal results for third-order tensors and other odd-order tensors. For example, for a third-order tensor in $\RTthree$, the most ``balanced'' unfolding of the tensor is a $d \times d^{2}$ matrix, and so directly reducing to an algorithm for agnostic matrix completion (e.g. \cite{koltchinskii2011nuclear}) would yield suboptimal $O(d^{2})$-type sample complexity.

Recent results of \cite{montanari2016spectral}, \cite{potechin2017exact}, and \cite{xia2017statistically} all give sub-$O(d^{2})$ type rates, but apply only to noiseless or well-specified models, and do not obviously extend to the agnostic setting. \cite{potechin2017exact} give \emph{exact completion} in the noiseless case after $\wt{O}(rd^{3/2})$ entries are observed. \cite{montanari2016spectral} show that a refined spectral approach based on unfolding can obtain sub-$O(d^{2})$ rates for prediction error in the noiseless setting, but they obtain a rate of $O(1/n^{2/3})$ that falls short of the $O(1/n)$-type rate we provide. Finally, \cite{xia2017statistically} recently showed that an algorithm based on power iteration provides $O(d/n)$-type rates once $n=\Omega(d^{3/2})$, but their result only applies to well-specified models, and it seems unlikely that this algorithmic approach succeeds in the fully agnostic setting.

Our results build on the seminal work of \cite{barak2016noisy} and \cite{potechin2017exact}, both of which use sum-of-squares to give $O(d^{3/2})$-type guarantees that improve on unfolding. The former was the first paper to study the agnostic setting, and gave slow excess risk guarantees for the absolute loss (i.e., rates growing as $\frac{1}{\sqrt{n}}$). Technically, our results build on their Rademacher complexity bounds for the \sos norms \citep{barak2016noisy}, as well as spectral bounds from \cite{hopkins2015tensor}. Obtaining fast rates, however, requires developing new technical tools and necessitates that we control the subgradient of the \sos norms. Our analysis here builds on ideas used to construct dual certificates for tensor completion in \cite{potechin2017exact}. 

Lastly, we mention that various recent works have begun to explore that power of \sos in other agnostic learning settings. Notably, \cite{klivans2018efficient} provide square loss risk bounds for \sos algorithms for robust regression. To the best of our knowledge our work is the first to provide \emph{fast rates} for \sos algorithms in any agnostic setting.

\subsection{Preliminaries}
\label{sec:preliminaries}
We let $\nrm*{\cdot}_{p}$ denote the $\ls_p$ norm, i.e. if $x\in\bbR^{d}$ is a vector then $\nrm*{x}_{p}=\prn*{\sum_{i=1}^{d}\abs*{x_i}^{p}}^{1/p}$. For a matrix $A$ we let $\nrm*{A}_{\specnorm}$ denote the operator norm/spectral norm and let $\nrm*{A}_{\nuc}$ denote the nuclear norm. For matrices or tensors we let $\nrm*{\cdot}_{F}$ denote the element-wise $\ls_2$ norm. For any norm $\nrm*{\cdot}$ we let $\nrm*{\cdot}_{\star}$ denote the dual.

We use $\wt{O}$ to suppress factors logarithmic in $d$, $r$, and $1/\delta$, where $\delta$ is the failure probability. 

\paragraph{Tensor notation.} The outer product between two vectors $u\in\bbR^{d_1}$ and $v\in\bbR^{d_2}$
is denoted $u\tens{}v$, and belongs to the space $\bbR^{d_1}\tens\bbR^{d_2}$. For a given vector $u$, we
write $u^{\tens{}k} = u\tens\cdots\tens{}u$ ($k$ times),
and likewise define $(\bbR^{d})^{\tens{}k}=\bbR^{d}\tens\cdots\tens\bbR^{d}$. This paper develops algorithms for completion of 3-tensors in $\RTthree$, which we frequently identify with elements of $\bbR^{d}\times{}\bbR^{d}\times{}\bbR^{d}$. In more detail, for a tensor $T\in(\bbR^{d})^{\tens{}3}$, we let
$T_{i,j,k}$ be such that $T = \sum_{i,j,k}T_{i,j,k}\cdot{}e_i\tens{}e_j\tens{}e_k$. For a pair of matrices $A,B$, we let $A\tens{}B$ denote the Kronecker
product, which obeys the relation $(A\tens{}B)(C\tens{}D)=(AC)\tens{}(BD)$. Given matrices $A_1,A_2,A_3$ and tensor $T=\sum_{i}u_i\tens{}v_i\tens{}w_i$, we define $(A_1\tens{}A_2\tens{}A_3)T = \sum_{i}(A_{1}u_i)\tens(A_{2}v_i)\tens{}(A_3w_i)$. Whenever $T$ is an orthogonal tensor of the form \pref{eq:orthogonal}, we let $r(T)$ denote the rank.

%%% Local Variables:
%%% mode: latex
%%% TeX-master: "paper"
%%% End:

\section{Subgradient of the \soslong nuclear norm}
\label{sec:subgradient}
% !TEX root = paper.tex
\subsection{Tensor nuclear norm and \soslong relaxation}
 In the classical results on matrix completion
 \citep{candes2009exact,candes2010power,recht2011simpler,koltchinskii2011nuclear},
 a central object is the \emph{nuclear} norm, which arises as a convex
 relaxation of the rank. A natural candidate to develop efficient
 algorithms for tensor completion is the \emph{tensor nuclear norm}
 (e.g. \cite{hillar2013most,friedland2018nuclear}), which may be defined via
\begin{equation}
\label{eq:tensor_nuclear_norm}
\nrm*{T}_{\nuc} = \inf\crl*{\sum_{i=1}^{r}\abs*{\lambda_i}:W=\sum_{i=1}^{r}\lambda_iu_i\tens{}v_i\tens{}w_i,\; \nrm*{u_i}_{2}=\nrm*{v_i}_2=\nrm*{w_i}_2=1,\; r\in\bbN},
\end{equation}
and whose dual is the \emph{injective tensor norm} $\nrm*{X}_{\inj}=\sup_{\nrm*{x}_{2}=\nrm*{y}_2=\nrm*{z}_2=1}\tri*{X,\vecone\tens\vectwo\tens\vecthree}$. Unfortunately, the optimization problem here---maximizing a degree-three polynomial over the sphere---is intractable in general. 

The approach we take, following \cite{barak2016noisy} and \cite{potechin2017exact}, is to employ the sum-of-squares hierarchy of convex relaxations \citep{shor1987approach,parrilo2000structured,lasserre2001global} which provides an increasingly tight sequence of relaxations of the optimization problem in \pref{eq:tensor_nuclear_norm}. To describe the relaxations, we require the notion of a \emph{pseudodistribution}.
\begin{definition}[Pseudodistribution (\cite{barak2016sumofsquares})]
Let $\mu:\bbR^{d}\to{}\bbR$ be a finitely supported function and let $\Ep_{\mu}f=\sum_{x\in\mathrm{supp}(\mu)}\mu(x)f(x)$. $\mu$ is said to be a degree-$k$ pseudodistribution if $\Ep_{\mu}1=1$ and $\Ep_{\mu}f^{2}\geq{}0$ for all polynomials $f$ of degree at most $k/2$.

\end{definition}
Given a degree-$s$ pseudodistribution $\mu$ and sytem of polynomial inequalities
$\cA = \crl*{f_1\geq{}0,\ldots,f_m\geq{}0}\cup\crl*{g_{1}=0,\ldots,g_{m'}=0}$, we write
$\mu\entails\cA$ if for all $S\subseteq{}\brk*{m}$ and all sum-of-squares
polynomials $h$ such that $\deg{}h +
\sum_{i\in{}S}\deg{}f_i\leq{}s$,
\begin{equation}
\Ep_{\mu}\brk*{h\prod_{i\in{}S}f_i} \geq{} 0,
\end{equation}
and $\Ep_{\mu}\brk*{g_iq}=0$ for all $i\in\brk*{m'}$ and all polynomials $q$ such that $\deg(g_iq)\leq{}\ls$.

With the pseudodistribution formalism, we define the degree-$k$ \sos injective norm as follows:\footnote{There are minor technical differences (e.g., scaling) between the \sos injective/nuclear norm definitions we use and those of \cite{barak2016noisy,potechin2017exact}.}
\begin{equation}
\label{eq:sos_injective}
  \nrm*{X}_{\injt_{k}}=\sup_{\substack{\mu\,\,\,\text{degree-$k$} \\
      \Ep_{\mu}\nrm*{x}_{2}^{2}=\Ep_{\mu}\nrm*{y}_{2}^{2}=\Ep_{\mu}\nrm*{z}_{2}^{2}=1}}\tri*{X,\Ep_{\mu}\brk*{\vecone\tens\vectwo\tens\vecthree}}.
\end{equation}
The degree-k \sos nuclear norm is simply defined as the dual:
$\nrm*{T}_{\nuct_k}\ldef{}\sup_{X:\nrm*{X}_{\injt_k}\leq{}1}\tri*{X,T}$. It
can equivalently be expressed by defining\footnote{This equivalence is proven in \pref{app:preliminaries} for completeness.}
\begin{equation}
\cK_{k} = \crl*{T\in\bbR^{d\times{}d\times{}d}\mid \exists{}\mu\, \text{degree-$k$}\,  : \Ep_{\mu}\brk*{\vecone\tens\vectwo\tens\vecthree}=T,\;\Ep_{\mu}\nrm*{x}_{2}^{2}=\Ep_{\mu}\nrm*{y}_{2}^{2}=\Ep_{\mu}\nrm*{z}_{2}^{2}=1 
},\label{eq:sos_nuclear}
\end{equation}
and then
\begin{equation}
\label{eq:sos_nuclear}
\nrm*{T}_{\nuct_{k}}=\inf\crl*{\alpha\mid{}T\in\cK_k/\alpha}.
\end{equation}
The \sos nuclear norm and injective norm can be evaluated in
$d^{O(k)}$ time \citep{grotschel1981ellipsoid,barak2016sumofsquares}. Moreover, the norms obey the ordering $\nrm*{T}_{\nuc}\geq\ldots\geq{}\nrm*{T}_{\nuct_k}\geq\ldots\geq{}\nrm*{T}_{\nuct_2}\geq\nrm*{T}_{\frobnorm}$, and likewise $\nrm*{X}_{\inj}\leq\ldots\leq{}\nrm*{X}_{\injt_k}\leq\ldots\leq{}\nrm*{X}_{\injt_2}\leq\nrm*{X}_{\frobnorm}$.

\subsection{Subgradient and norm compatibility}
\label{sec:subgradient_norm}
Our algorithms are based on empirical risk minimization with \sos
nuclear norm constraints, i.e., algorithms that minimize the empirical
loss over the set $\cT=\crl*{T\in\RTthree\mid\nrm*{T}_{\nuct_k}\leq{}\tau}$ for appropriate choice of $\tau$ and $k$. The technical challenge to analyzing this type of relaxation is that even if measurements are realized by a rank-$r$ tensor, there is nothing that guarantees a-priori that the tensor $\term$ output by the algorithm is itself low rank. 
Letting $T$ be a rank-$r$ orthogonal tensor, our main technical result here shows that if $\tau=\nrm*{T}_{\nuct_{k}}$, then for all elements of $T'\in\cT$, the error $\Delta=T'-T$ is ``approximately low rank'' in a sense that suffices to guarantee good generalization performance.
\begin{theorem}[Formal version of \pref{thm:norm_comparison_informal}]
\label{thm:sos_norm_comparison}
Let $k\geq{}4$. Let $T\in\RTthree$ be an orthogonal rank-$r$ tensor,
and let $T'\in\RTthree$ be an arbitrary tensor with $\Delta\ldef{}T'-T$. If
$\nrm*{T'}_{\nuct_{k+2}}\leq{}\nrm*{T}_{\nuct_{k+2}}$, then
\begin{equation}
  \label{eq:sos_to_frobenius}
  \nrm*{\Delta}_{\nuct_k} \leq{} 68r\cdot\nrm*{\Delta}_{\frobnorm}.
\end{equation}
\end{theorem}
In the analysis of nuclear norm regularization for matrix completion---and more broadly, throughout high-dimensional statistics---the key tool used to establish guarantees along the lines of \pref{eq:sos_to_frobenius} is a characterization for the \emph{subgradient} for the nuclear norm, and the related notion of \emph{decomposability} \citep{negahban2012unified,negahban2012restricted}. It is known clasically \citep{watson1992characterization} that for any matrix $W$ with singular value decomposition $W=U\Sigma{}V^{\trn}$,
\begin{equation}
\label{eq:matrix_subgradient}
\partial\nrm*{W}_{\nuc}=\crl*{UV^{\trn} + X\mid{}U^{\trn{}}X=XV=0,\nrm*{X}_{\specnorm}\leq{}1}.
\end{equation}
Our approach in the remainder of this section is to establish a similar result for the subgradient of the \sos nuclear norm $\nrm*{\cdot}_{\nuct_k}$ at any orthogonal tensor $T$. From here \pref{thm:sos_norm_comparison} will quickly follow.

As a first step, we need to define certain subspaces and projection operators associated with $T$.
\paragraph{Subspaces}
For the remainder of the section we let $T$ be a rank-$r$ orthogonal
tensor as in \pref{eq:orthogonal}. Define $\bbU=\mathrm{span}\crl*{u_i}$, $\bbV=\mathrm{span}\crl*{v_i}$, and $\bbW=\mathrm{span}\crl*{w_i}$, and note that each subspace has dimension at most $r$. Let $\cP_{\bbU}:\bbR^{d}\to\bbR^{d}$ and
$\cP_{\bbV}:\bbR^{d}\to\bbR^{d}$ and $\cP_{\bbW}:\bbR^{d}\to\bbR^{d}$
be orthogonal projections onto these subspaces and
$\cP_{\bbU^{\perp}}$, $\cP_{\bbV^{\perp}}$ and $\cP_{\bbW^{\perp}}$ be
the projections onto the respective orthogonal complements. We define
projection operators from $\RTthree$ to $\RTthree$ for all $2^{3}$ combinations of
subspaces:\footnote{Following the convention in
  \pref{sec:preliminaries}, if $X=\sum_{i}a_i\tens{}b_i\tens{}c_i$,
  then
  $(\cP_{\bbU}\tens\cP_{\bbV}\tens\cP_{\bbW})X=\sum_{i}(\cP_{\bbU}a_i)\tens{}(\cP_{\bbV}b_i)\tens{}(\cP_{\bbW}c_i)$. It is also useful to note that for any $x,y,z$ we have $\tri*{(\cP_{\bbU}\tens\cP_{\bbV}\tens\cP_{\bbW}),x\tens{}y\tens{}z}=\tri*{X,(\cP_{\bbU}x)\tens{}(\cP_{\bbV}y)\tens{}(\cP_{\bbW}z)}$.}
\begin{equation}
\label{eq:all_subspaces}
\begin{aligned}
&\cQ^{0}_{T^{\parallel}} = \cP_{\bbU}\tens\cP_{\bbV}\tens\cP_{\bbW},
&&\cQ^{0}_{T^{\perp}} = \cP_{\bbU^{\perp}}\tens\cP_{\bbV^{\perp}}\tens\cP_{\bbW^{\perp}},\\
&\cQ^{1}_{T^{\parallel}} = \cP_{\bbU^{\perp}}\tens\cP_{\bbV}\tens\cP_{\bbW},
&&\cQ^{1}_{T^{\perp}} = \cP_{\bbU}\tens\cP_{\bbV^{\perp}}\tens\cP_{\bbW^{\perp}},\\
&\cQ^{2}_{T^{\parallel}} = \cP_{\bbU}\tens\cP_{\bbV^{\perp}}\tens\cP_{\bbW},
&&\cQ^{2}_{T^{\perp}} = \cP_{\bbU^{\perp}}\tens\cP_{\bbV}\tens\cP_{\bbW^{\perp}},\\
&\cQ^{3}_{T^{\parallel}} = \cP_{\bbU}\tens\cP_{\bbV}\tens\cP_{\bbW^{\perp}},
&&\cQ^{3}_{T^{\perp}} = \cP_{\bbU^{\perp}}\tens\cP_{\bbV^{\perp}}\tens\cP_{\bbW}.
\end{aligned}
\end{equation}
Lastly, we define two subspaces that play a central role in our analysis:
\begin{equation}
\label{eq:tensor_subspaces}
\cQ_{T^{\para}} = \cQ_{T^{\para}}^{0} + \cQ_{T^{\para}}^{1}  + \cQ_{T^{\para}}^{2} + \cQ_{T^{\para}}^{3},\quad\text{and}\quad
\cQ_{T^{\perp}} = \cQ_{T^{\perp}}^{0} + \cQ_{T^{\perp}}^{1}  + \cQ_{T^{\perp}}^{2} + \cQ_{T^{\perp}}^{3}.
\end{equation}
One can verify via multilinearity that $X = \cQ_{T{\para}}(X) +
\cQ_{T^{\perp}}(X)$ for all $X\in\RTthree$, and that any tensor in the range of
$\cQ_{T^{\para}}$ spans at most $r$ dimensions along at least two
modes. We can now state our main theorem for the subgradient.
\begin{theorem}[Subgradient of \sos nuclear norm]
\label{thm:sos_subgradient}
Let $k\geq{}4$, and let $T=\sum_{i=1}^{r}\lambda_i\cdot{}u_i\tens{}v_i\tens{}w_i$ be an
orthogonal rank-$r$ tensor. Define
$X^{\star}=\sum_{i=1}^{r}u_i\tens{}v_i\tens{}w_i$. Then for all $X\in\RTthree$ with
$\nrm*{X}_{\injt_k}\leq{}1/64$, and for all $T'\in\RTthree$, it holds that
  \begin{equation}
    \label{eq:sos_subgradient}
    \nrm*{T'}_{\nuct_{k+2}} \geq{}  \nrm*{T}_{\nuct_{k+2}} + \tri*{X^{\star} + \QTperp(X),T'-T}.
  \end{equation}
\end{theorem}
In other words, \pref{thm:sos_subgradient} states that
\[
\crl*{X^{\star} + \QTperp(X)\mid{}\nrm*{X}_{\injt_k}\leq{}1/64}
\subset \partial\nrm*{T}_{\nuct_{k+2}},
\]
which we may view as a generalization of the matrix
subgradient characterization
\pref{eq:matrix_subgradient}. \cite{yuan2016tensor} proved a similar
result for the (exact) tensor nuclear norm. Our proof of
\pref{thm:sos_subgradient} 
% (deferred to \pref{app:subgradient})
shows that the essence of their proof can be captured by a low-degree sum-of-squares proof. It builds on the approach introduced in \cite{potechin2017exact} to provide dual certificates for exact tensor completion.

With the subgradient lemma in hand, the path to the ``approximately low rank'' result of \pref{thm:sos_norm_comparison} is clear. Suppose that $\nrm*{T'}_{\nuct_{k+2}}\leq{}\nrm*{T}_{\nuct_{k+2}}$, and let $\Delta=T'-T$. By appropriately choosing the dual tensor $X$ in \pref{eq:sos_subgradient}, we can show that
\begin{equation}
  \label{eq:sos_cone}
  \nrm*{\QTperp(\Delta)}_{\nuct_k} \leq{} 64\cdot\nrm*{\QTpara^{0}(\Delta)}_{\nuct_{k}}.
\end{equation}
which implies that $\nrm*{\Delta}_{\nuct_k}\leq{}65\nrm*{\QTpara(\Delta)}_{\nuct_k}$
The final result follows because
$\nrm*{\QTpara(\Delta)}_{\nuct_k}\approxleq{}r\nrm*{\Delta}_{\frobnorm}$,
which is a consequence of the earlier remark that all the projections
used to define $\QTpara$ in \pref{eq:tensor_subspaces} project into
$r$ dimensions along at least two modes. This full argument is in
\pref{app:subgradient}.

%%% Local Variables:
%%% mode: latex
%%% TeX-master: "paper"
%%% End:

\section{Agnostic tensor completion}
\label{sec:main_results}
% !TEX root = paper.tex

We now state our main learning results, which use the \sos nuclear norm to give efficient algorithms with fast rates for agnostic tensor completion and tensor sensing. For both results we receive observations $(X_1,Y_1),\ldots,(X_n,Y_n)$ i.i.d. according to an unknown distribution $\cD$, where $(X_t,Y_t)\in\RTthree\times{}\bbR$, and the goal is to obtain low square loss excess risk in the sense of equation \pref{eq:excess_risk}. We let $\Eh$ denote the empirical expectation,  which is uniform over the examples $\crl*{(X_t,Y_t)}_{t=1}^{n}$.

\subsection{Tensor completion}
\label{sec:tensor_completion}
In the tensor completion model we take observations $X_t$ to be of the
form $X_{t}=e_{i_t}\tens{}e_{j_t}\tens{}e_{k_t}$, where
$(i_t,j_t,k_t)\in\brk*{d}^{3}$ is selected uniformly at
random.\footnote{Note that we sample entries with replacement, whereas
  related works use without-replacement sampling
  \citep{barak2016noisy,potechin2017exact}.} In the noiseless or well-specified
setting, this corresponds to observing a single entry of an unknown
tensor, but we make no assumption on the responses $Y_t$ other than
boundedness. The main theorem is as follows. 
\begin{theorem}[Formal version of \pref{thm:tensor_completion_informal}]
  \label{thm:tensor_completion}
 Let $\tau>0$ be fixed. Suppose that $\abs*{Y}\leq{}R$
almost surely, and let $\term$ be the empirical risk minimizer over
the tensor class $\cT=\crl*{T\in\RTthree\mid{}\nrm*{T}_{\nuct_6}\leq{}\tau,\nrm*{T}_{\infty}\leq{}R}$. Then for all $n\leq{}d^{3}$, with probability at least $1-\delta$, 
\begin{equation}
  \label{eq:tensor_completion_risk}
\poprisk(\term) - \poprisk(\ts) \leq{}
O\prn*{
\frac{R^{2}r^{2}(\ts) d^{3/2}\log^{6}d}{n}+\frac{R^{2}\log(1/\delta)}{n} 
}
% O\prn*{
% R^{2}\cdot\frac{r^{2}d^{3/2}\log^{6}d + \log(1/\delta)}{n}
% },
\end{equation}
for all orthogonal tensors $T^{\star}\in\RTthree$ with $\nrm*{T^{\star}}_{\nuct_6}=\tau$
and $\nrm*{T^{\star}}_{\infty}\leq{}R$.
\end{theorem}
Let us spend a moment interpreting the theorem. First, let $T^{\star}=\argmin_{T:\textnormal{rank-$r$}}\poprisk(T)$. Then, by setting $\tau=\nrm*{T^{\star}}_{\nuct_6}$, we are guaranteed that with probability at least $1-\delta$,
\begin{equation}
\label{eq:tensor_completion_excess_risk}
\poprisk(\term) - \inf_{T:\textnormal{rank-$r$}}\poprisk(T) \leq{}
O\prn*{
\frac{R^{2}r^{2} d^{3/2}\log^{6}d}{n}+\frac{R^{2}\log(1/\delta)}{n} 
}.
\end{equation}
More generally, \pref{eq:tensor_completion_risk} implies an \emph{exact oracle inequality} \citep{koltchinskii2011nuclear,gaiffas2011sharp}: With probability at least $1-\delta$, we have
\[
\poprisk(\term) \leq{} \inf_{T:\nrm*{T}_{\nuct_6}=\tau} \crl*{\poprisk(T) + 
O\prn*{
\frac{R^{2}r^{2}(T) d^{3/2}\log^{6}d}{n}+\frac{R^{2}\log(1/\delta)}{n} 
}
}.
\]
%where the right-hand side crucially depends only on the rank of the benchmark.

Let us compare the result in detail with \cite{barak2016noisy}, which is the only other polynomial time agnostic tensor completion result with sub-$O(d^{2})$ sample complexity. For general noise distributions, their analysis gives an excess risk bound that scales as $\wt{O}\prn*{\sqrt{\frac{r^{2}d^{3/2}}{n}}
}$. The bound in \pref{eq:tensor_completion_risk} matches this dependence on all the parameters, but is squared, and is thus always tighter. The result has excess risk against arbitrary tensors, however, while our bound requires orthogonality of the benchmark. We do not know whether this restriction can be removed.\vspace{3pt}

Interestingly, while \cite{barak2016noisy} give excess risk bounds against incoherent tensors, we do not require incoherence. This is because we control the complexity of the benchmark through the $\ls_{\infty}$ norm of the entries rather than through the Frobenius norm; this parallels the situation in the matrix setting \citep{koltchinskii2011nuclear,gaiffas2011sharp}. Applying the spectral bounds of \cite{barak2016noisy} without incoherence requires slightly tightening the analysis.\vspace{3pt}

Lastly, we remark that the guarantee \pref{eq:tensor_completion_excess_risk} requires setting the parameter $\tau$ based on the norm of the unknown benchmark $T^{\star}$. It is likely that this can be relaxed by appealing to penalized empirical risk minimization rather than empirical risk minimization as in \cite{koltchinskii2011nuclear}, but we leave this for future work.

\subsection{Tensor sensing}
In this setting we give agnostic learning guarantees for a setting we
call \emph{tensor sensing}, which generalizes the matrix compressed
sensing setup studied in \cite{negahban2011estimation}. We assume that
observations $X\in\RTthree$ have independent entries from
$\cN(0\midsem{}1)$ and---as in the tensor completion setting---allow
$Y\in\bbR$ to be arbitrary. For each tensor $T$, define $R(T)$ to be
the smallest almost-sure bound on $\abs*{\tri*{T,X}-Y}$. As in the tensor
completion setup, the main result is a fast rate with $\wt{O}\prn*{\frac{r^{2}d^{3/2}}{n}}$-type scaling.
%\dfcomment{cite matrix sensing results, eg in wainwright's book or negahban}
\begin{theorem}
  \label{thm:tensor_sensing}
Let $\tau>0$ be fixed, and let $\term$ be the empirical risk minimizer over
the tensor class $\cT=\crl*{T\in\RTthree\mid{}\nrm*{T}_{\nuct_6}\leq{}\tau}$.
% Then for any orthogonal tensor $T^{\star}\in\RTthree$ with
% $\nrm*{T^{\star}}_{\nuct_6}=\tau$, as soon as $n=\Omega(r^{2}(\ts)
% d^{3/2}\log^{1/2}d)$, we have that with probability at least
% $1-O(d^{-10})$,
Then with probability at least $1-\delta$,
\begin{equation}
  \label{eq:tensor_sensing_risk}
\poprisk(\term) - \poprisk(\ts) \leq{} O\prn*{
\frac{R^{2}(T^{\star})r^{2}(\ts) d^{3/2}\log^{3}(d/\delta)}{n}
}.
\end{equation}
for all orthogonal tensors $T^{\star}\in\RTthree$ for which
$n=\Omega(r^{2}(\ts) d^{3/2}\log^{1/2}d+\log(1/\delta))$ and $\nrm*{T^{\star}}_{\nuct_6}=\tau$.
\end{theorem}

%%% Local Variables:
%%% mode: latex
%%% TeX-master: "paper"
%%% End:

\label{sec:techniques}
%\paragraph{Overview of techniques}
% !TEX root = paper.tex
\paragraph{Overview of analysis}
We now sketch how the subgradient theorem can be combined with empirical process arguments to prove \pref{thm:tensor_completion} and \pref{thm:tensor_sensing}. We follow a generic recipe given in \pref{app:generalization}---specifically, \pref{thm:generic_offset}---which shows that to control the generalization error of empirical risk minimization, it suffices to bound a certain ``offset'' or ``shifted'' empirical process. For any fixed benchmark $T^{\star}$, the excess risk relative to $T^{\star}$ is bounded as
\begin{equation}
\label{eq:offset_process}
\poprisk(\term) - \poprisk(T^{\star})
\leq{} \sup_{\Delta\in\cT-T^{\star}}\crl*{ (\En-\Eh)\brk*{2\tri*{\Delta,X}(\tri*{\ts,X}-Y)} +
  \En\tri*{\Delta,X}^{2} - 2\Eh\tri*{\Delta,X}^{2}}.
\end{equation}

The offset process on the right-hand side was used to obtain high-probability fast rates for misspecified models by \cite{liang2015learning}, and its analysis is closely related to that of \cite{lecue2013learning,mendelson2014learning}. To bound the process, it suffices to establish a type of lower isometry/restricted eigenvalue property, which we state here for the case of tensor regression: let $\dataop:\RTthree\to\bbR^{n}$ be the data operator, which maps any tensor $T$ to the sequence $\tri*{T,X_1},\ldots,\tri*{T,X_n}$, and let $\Sigma=\En_{X}\brk*{XX^{\trn}}\in\bbR^{d^{3}\times{}d^{3}}$ be the covariance matrix for the vectorized measurements. Then it suffices to show that with high probability, the following \emph{restricted eigenvalue} bound holds:
\[
\frac{1}{\sqrt{n}}\nrm*{\dataop(\Delta)}_{2} \geq{} c\nrm{\Sigma^{1/2}\Delta}_{\frobnorm}\quad\forall{}\Delta\in\cT-T^{\star},
\]
where $c>1/\sqrt{2}$ is a sufficiently large constant.

Our starting point to establish the guarantee is to borrow a bound from \cite{hopkins2015tensor}, which states that $\En\nrm*{X}_{\injt_4}=O(d^{3/4}\log^{1/4}d)$ under gaussian measurements, and suffices to bound the Rademacher complexity of our tensor class. Using this bound in conjunction with standard gaussian concentration arguments and the ``peeling'' method (e.g. \citep{negahban2011estimation}), we prove \pref{thm:tensor_sensing_re}, which states that with high probability,
\[
\frac{1}{\sqrt{n}}\nrm*{\dataop(\Delta)}_{2} \geq{}  0.79\cdot\nrm*{\Delta}_{\frobnorm}-\frac{Cd^{3/4}\log^{1/4}d}{\sqrt{n}}\cdot\nrm*{\Delta}_{\nuct_{4}}\quad\quad\forall{}\Delta\in\RTthree.
\]
Combined with \pref{thm:sos_norm_comparison}, which asserts that all $\Delta\in\cT-T^{\star}$ have $\nrm*{\Delta}_{\nuct_4}\leq{}O(r(T^{\star}))\cdot\nrm*{\Delta}_{2}$, we have the following consequence: once $n=\Omega(r^{2}(T^{\star})d^{3/2}\log^{1/2}d)$, with high probability,
\[
\frac{1}{\sqrt{n}}\nrm*{\dataop(\Delta)}_{2} \geq{}  (0.79-o(1))\cdot\nrm*{\Delta}_{\frobnorm}\quad\quad\forall{}\Delta\in\cT-T^{\star}.
\]

To establish the analogous bound in the tensor completion model we use the \sos Rademacher bound from \cite{barak2016noisy}, but utilize somewhat different concentration arguments. Indeed, due to the sparse nature of the measurement distribution one cannot hope to exactly establish the restricted eigenvalue property for $\dataop$, and must instead show that it holds up to a small additive error.

%%% Local Variables:
%%% mode: latex
%%% TeX-master: "paper"
%%% End:

\section{Computational lower bounds}
\label{sec:lower_bounds}
% !TEX root = paper.tex

In the rank-one case, the excess risk bound of \pref{thm:tensor_completion} scales as $\wt{O}(d^{3/2}/n)$, while the excess risk attained by the natural inefficient algorithm scales as $\wt{O}(d/n)$. It is natural to ask whether this $O(d^{1/2})$ gap can be improved or whether it poses a fundamental barrier. In the slow rate regime, \cite{barak2016noisy} gave a computational lower bound showing that finding efficiently computable classes of tensors for which the Rademacher complexity grows as $o(\sqrt{d^{3/2}/n})$ is at least as hard as refuting random instances of \xor with $o(d^{3/2})$ clauses. In this section we show that this computational hardness is also present in the fast rate regime: Under plausible average-case hardness assumptions, no polynomial time algorithm can obtain a fast rate for square loss scaling as $O(d^{3/2-\veps}/n)$ for any $\veps>0$.

Our improper learning lower bound applies to \emph{any algorithm} that obtains low excess risk in the sense of \pref{eq:tensor_completion_risk}, and states that under conjectured hardness of a certain distinguishing problem for \xor it is not possible to improve the $O(d^{3/2})$ dependence on dimension.

We reduce from the \xor problem over variables $x\in\pmo^{d}$. A \xor instance consists of a sequence of $m$ clauses of the form 
\[
x_i\cdot{}x_j\cdot{}x_k = z_{ijk},
\]
where $z_{ijk}\in\pmo$ is a target. We consider two families of instances:
\begin{itemize}
\item \emph{Planted.} Fix an arbitrary assignment $a\in\pmo^{d}$. Select $m$ triples $(i,j,k)$ uniformly at random with replacement.\footnote{We work in the slightly non-standard with-replacement model to simplify the mapping onto the with-replacement tensor completion model in \pref{thm:tensor_completion}.} For each such triple $(i,j,k)$, include a clause
\[
x_i\cdot{}x_j\cdot{}x_k = z_{ijk}\ldef\left\{\begin{array}{ll}
a_i\cdot{}a_j\cdot{}a_k,\quad&\text{with probability $1-\eta$}\\
-a_i\cdot{}a_j\cdot{}a_k,\quad&\text{with probability $\eta$.}
\end{array}\right.
\]
Note, we sample the value $z_{i,j,k}$ for a triple $(i,j,k)$ only once: if the triple is sampled multiple times, the value $z_{ijk}$ will be the same. 

\item \emph{Random.} Select $m$ triples $(i,j,k)$ uniformly at random with replacement, and take each clause to be $x_i\cdot{}x_j\cdot{}x_k=z_{ijk}$, where $z_{ijk}$ is drawn from $\pmo$ uniformly at random. Again, we sample the value $z_{i,j,k}$ for a triple $(i,j,k)$ only once. 

\end{itemize}
An algorithm for the distinguishing problem takes $m$ clauses as input and outputs either ``Planted'' or ``Random''. The algorithm is said to succeed if it outputs ``Planted'' for planted instances and ``Random'' for random instances with probability at least $1-o(1)$ over the draw of the instance. Note that the problem becomes easier as $\eta$ gets smaller, and in particular when $\eta=0$ the problem can be solved in polynomial time using Gaussian elimination. 

\begin{conjecture}
\label{conj:3xor}
There is some constant $\eta<1/4$ such that no algorithm that succeeds for the \xor distinguishing problem with $m=o(d^{3/2})$ runs in polynomial time.
\end{conjecture}
All known polynomial time algorithms for distinguishing require $m=\Omega(d^{3/2})$ clauses, and conjectured hardness of the closely related problem of strong refutation for random \xor for with $o(d^{3/2})$ clauses has been used as a basis to establish hardness of other learning problems \citep{daniely2016complexity,raghavendra2017strongly,kothari2017sum,feldman2018complexity}. 

\begin{theorem}
\label{thm:lower_bound}
Let $\veps>0$ be fixed. Assuming the \xor distinguishing conjecture, there is no polynomial time algorithm for agnostic tensor completion that guarantees that for any distribution $\cD$, with probability at least $1-o(1)$,
\begin{equation}
\label{eq:lower_bound_risk}
\poprisk(\term) - \inf_{\ts:\;\textnormal{rank-$1$}}\poprisk(\ts) = O\prn*{\frac{d^{3/2-\veps}}{n}}.
\end{equation}
\end{theorem}

%%% Local Variables:
%%% mode: latex
%%% TeX-master: "paper"
%%% End:

\section{Conclusion}
Our results demonstrate the power of the
sum-of-squares hierarchy for agnostic statistical learning, and show that \soslong algorithms can obtain fast rates for prediction with the square loss. We hope our work will serve as a
starting point for applying sum-of-squares to obtain polynomial time
algorithms with fast rates in statistical learning for broader classes of models.

A few immediate technical questions emerge. Can the
dependence on rank in our results be improved? Can the
subgradient results be extended to the general undercomplete or even
overcomplete case? Can similar agnostic learning results be obtained
with a more practical algorithm that does not rely on solving the full sum-of-squares SDP?
\paragraph{Acknowledgements} We thank Sasha Rakhlin and Ankur Moitra for
helpful discussions and thank Matthew J. Telgarsky for being a
constant source of inspiration.

%\newpage
\bibliography{../refs}
%\newpage

\appendix

\section{Preliminaries}
\label{app:preliminaries}
% !TEX root = paper.tex

\subsection{Sum-of-squares proof system}
Let $\cA=\crl*{f_1\geq{}0,\ldots,f_m\geq{}0}\cup\crl*{g_{1}=0,\ldots,g_{m'}=0}$
be a system of polynomial constraints. A \emph{degree-$\ls$ sum-of-squares proof} that 
$\cA$ implies a constraint $\crl*{h\geq{}0}$ is a set of polynomials
$(q_i)_{i\in\brk*{m'}}$ and sum-of-squares
polynomials $(p_{S})_{S\subseteq\brk*{m}}$ such that
\[
h = \sum_{S\subseteq\brk*{m}}p_S\prod_{i\in{}S}f_i + \sum_{i\in\brk*{m'}}q_ig_i,
\]
and where and $\deg(p_S\prod_{i\in{}S}f_i)\leq{}\ls$ for all $S$ and
$\deg(q_ig_i)\leq{}\ls$ for all $i$.\footnote{We use the convention
  $\prod_{i\in\emptyset}f_i=1$, so that if $\provable_{\ls}h$ then $h$ itself is a degree-$\ls$ sum of squares.} We write
$\cA\provable_{\ls}\crl*{h\geq{}0}$
whenever such a proof exists.
Our proofs going forward utilize the well-known duality of \sos proofs
and pseudodistributions. See \cite{odonnell2013approximability} and
\cite{barak2016sumofsquares} for further discussion, as well as inference rules for the SoS proof
system.

We note the following well-known, but useful lemma: 
\begin{lemma}[Pseudo-Cauchy Schwarz]
  \label{lem:pseudo_cs}
Let $f$ and $g$ be polynomials and let
$\ls=2\max\crl*{\deg{}f,\deg{}g}$. Then for any $\eta>0$,
\begin{equation}
\provable_{\ls}\crl*{fg \leq{} \frac{\eta}{2}f^{2} +
  \frac{1}{2\eta}g^{2}}.\label{eq:cs_poly}
\end{equation}
As a consequence, if $\mu$ is a degree-$s$ pseudodistribution with $s\geq{}\ls$, then

\begin{equation}
\Ep_{\mu}\brk*{fg}\leq{}\sqrt{\Ep_{\mu}\brk*{f^{2}}\cdot\Ep_{\mu}\brk*{g^{2}}}.\label{eq:cs_pseudo}
\end{equation}
\end{lemma}

\subsection{Basic technical results}

\subsubsection{Sum-of-squares norms}

We state a few lemmas capturing useful properties of the sum-of-squares norms. 

\begin{proposition}
\label{prop:sos_dual}
The SoS nuclear norm and SoS injective norm are dual: $\nrm*{\cdot}_{\injt_r^{\star}}=\nrm*{\cdot}_{\nuct_r}$ and $\nrm*{\cdot}_{\nuct_r^{\star}}=\nrm*{\cdot}_{\injt_r}$.
\end{proposition}
\begin{proof}[\pfref{prop:sos_dual}]
It is immediate from the norm definitions that
\[
\nrm*{X}_{\nuct_r^{\star}}=\sup_{W\in\cK_r}\tri*{W,X} = \sup_{\substack{\mu\,\,\,\text{degree-$r$} \\ \Ep_{\mu}\nrm*{x}_{2}^{2}=\Ep_ {\mu}\nrm*{y}_{2}^{2}=\Ep _{\mu}\nrm*{z}_{2}^{2}=1\\W=\Ep_ {\mu}\brk*{x\tens{}y\tens{}z}}}\tri*{W,X} = \nrm*{X}_{\injt_r}.
\]
The other direction is a consequence of the standard
duality theory for finite-dimensional Banach spaces. See, e.g.,
Theorem 15.4 in \cite{rockafellar1970convex}.
\end{proof}

\begin{lemma}
\label{lem:nuclear_orthogonal}
Let $T = \sum_{i=1}^{r}\lambda_i\cdot{}u_i\tens{}v_i\tens{}w_i$ be an
orthogonal rank-$r$ tensor. Then for all $k\geq{}4$.
\begin{equation}
\nrm*{T}_{\nuct_k}=\nrm*{T}_{\nuc}=\sum_{i=1}^{r}\abs*{\lambda_i},\quad\text{and}\quad \nrm*{T}_{\injt_k}=\nrm*{T}_{\inj}=\max_{i\leq{}r}\abs*{\lambda_i}.
\end{equation}
\end{lemma}
% \begin{remark}
\pref{lem:nuclear_orthogonal} states that the SoS relaxations of the
nuclear norm and injective norm are essentially ``integral'' for
orthogonal tensors. Note, this should not be a huge surprise
since it is well-known that polynomial time methods such as power
iteration succeed at decomposing orthogonal tensors
\citep{kolda2009tensor}. We should mention it doesn't seem
possible to directly apply such results to give agnostic
learning guarantees along the lines of our main theorem. While
our benchmark is an orthogonal tensor, the data itself may have no
orthogonal structure, and thus there is no clear object to which one might apply such a decomposition.
% \end{remark}

\begin{proof}[\pfref{lem:nuclear_orthogonal}]
We may assume $\lambda_i\geq{}0$ without loss of generality. We first prove
equality for the injective norms. Let
$i^{\star}=\argmax_{i\leq{}r}\lambda_i$. As a starting point,
for any $k$ we have
$\nrm*{T}_{\injt_k}\geq{}\nrm*{T}_{\inj}\geq{}\lambda_{i^{\star}}$
by exhibiting $u_{i^{\star}}\tens{}v_{i^{\star}}\tens{}w_{i^{\star}}$
as a feasible solution to the supremum in $\nrm*{T}_{\inj}=\sup_{\nrm*{x}_2=\nrm*{y}_2=\nrm*{z}_2=1}\tri*{T,x\tens{}y\tens{}z}$.

For the upper bound, let $x,y,z$ be indeterminates and---exploiting orthogonality---let us change
coordinates such that $u_i=v_i=w_i=e_i$. Then we have
\[
\tri*{T,x\tens{}y\tens{}z} = \sum_{i=1}^{r}\lambda_ix_i{}y_iz_i
\]
From equation \pref{eq:cs_poly}, we have
\[
\provable_{4} x_iy_iz_i \leq{} \frac{1}{2}x_i^{2} +
\frac{1}{2}y_i^{2}z_{i}^{2}.
\]
We also have
$\provable_{4}\sum_{i=1}^{r}y_i^{2}z_i^{2}\leq{}(\sum_{i=1}^{r}y_i^{2})(\sum_{i=1}^{r}z_i^{2})$. 
By the additivity of SoS proofs, and since we have assumed
$\lambda_i\geq{}0$, this implies
\[
\provable_{4}
\tri*{T,x\tens{}y\tens{}z} \leq{} \max_{i\leq{}r}\lambda_i\cdot\prn*{\frac{1}{2}\nrm*{x}_{2}^{2} +
  \frac{1}{2}\nrm*{y}_{2}^{2}\nrm*{z}_{2}^{2}}.
\]
Now let $\cA=\crl*{\nrm*{x}_{2}^{2}=1, \nrm*{y}_{2}^{2}=1,
  \nrm*{z}_{2}^{2}=1}$. We claim
$\cA\provable_{4}\nrm*{y}_{2}^{2}\nrm*{z}_{2}^{2}=1$. To see this,
write $1-\nrm*{y}_{2}^{2}\nrm*{z}_{2}^{2} =
(1-\nrm*{y}_{2}^{2})(1+\nrm*{z}_{2}^{2})  + (1-\nrm*{z}_{2}^{2}) +
(\nrm*{y}_{2}^{2}-1)$ and use that $\deg((1-\nrm*{y}_{2}^{2})(1+\nrm*{z}_{2}^{2}))=4$.

Putting everything together, we see that
$\cA\provable_{4}\tri*{T,x\tens{}y\tens{}z}\leq{}\max_{i\leq{}r}\lambda_i$. Thus, since the $\ls_2$ norm is
preserved under change of basis, it follows that if $\mu$ is any feasible degree-4 pseudodistribution for
the maximization problem \pref{eq:sos_injective}, we must have $\Ep_{\mu}
\tri*{T,x\tens{}y\tens{}z}\leq{}\max_{i}\lambda_i$, and so
$\nrm*{T}_{\injt_k}\leq\max_{i}\lambda_i$.

We now establish equality for the nuclear norms. We trivially have
$\nrm*{T}_{\nuct_k}\leq{}\nrm*{T}_{\nuc}\leq{}\sum_{i=1}^{r}\lambda_i$
by exhibiting the decomposition $T =
\sum_{i=1}^{r}\lambda_i\cdot{}u_i\tens{}v_i\tens{}w_i$ as a feasible
solution to the minimization problem in \pref{eq:tensor_nuclear_norm}. For the other direction, define
$X^{\star}=\sum_{i=1}^{r}u_i\tens{}v_i\tens{}w_i$, and observe that
the equality we just established for the injective norm implies
$\nrm*{X^{\star}}_{\injt_k}\leq{}1$. Thus, using the duality of the SoS
nuclear norm and injective norm from \pref{prop:sos_dual}, we have
\[
\nrm*{T}_{\nuct_k}
=\sup_{X\in\RTthree:\nrm*{X}_{\injt_k}\leq{}1}\tri*{X,T}\geq{}\tri*{X^{\star},T}
= \sum_{i=1}^{r}\lambda_i,
\]
where the last equality uses that $\crl*{u_i}$, $\crl*{v_i}$, and
$\crl*{w_i}$ are all orthogonal.
\end{proof}

\begin{proposition}
  \label{prop:injective_ub}
Let $k\geq{}4$. For any degree-$k$ pseudodistribution $\mu$,
    \begin{equation}
      \label{eq:injective_ub_pseudo}
      \tri*{T,\Ep_\mu\brk*{x\tens{}y\tens{}z}}\leq\nrm*{T}_{\injt_k}\cdot\sqrt{\Ep_\mu\nrm*{x}^{2}_{2}\cdot \Ep_\mu\nrm*{y}^{2}_{2}\cdot \Ep_\mu\nrm*{z}^{2}_{2}}.
    \end{equation}
    Furthermore, the following statements hold:
    \begin{equation}
      \label{eq:injective_ub_sos}
      \begin{aligned}
        \crl*{\nrm*{z}_{2}^{2}\leq{}1} \provable_{k}
        \tri*{T,x\tens{}y\tens{}z} \leq{} \nrm*{T}_{\injt_k}\cdot\prn*{\frac{1}{2}\nrm*{x}_{2}^{2}
        + \frac{1}{2}\nrm*{y}_{2}^{2}},\\
        \crl*{\nrm*{y}_{2}^{2}\leq{}1} \provable_{k}
        \tri*{T,x\tens{}y\tens{}z} \leq{}
        \nrm*{T}_{\injt_k}\cdot\prn*{\frac{1}{2}\nrm*{x}_{2}^{2}
        + \frac{1}{2}\nrm*{z}_{2}^{2}},\\
        \crl*{\nrm*{x}_{2}^{2}\leq{}1} \provable_{k}
        \tri*{T,x\tens{}y\tens{}z} \leq{}
        \nrm*{T}_{\injt_k}\cdot\prn*{\frac{1}{2}\nrm*{y}_{2}^{2}
        + \frac{1}{2}\nrm*{z}_{2}^{2}}.
      \end{aligned}
    \end{equation}
\end{proposition}
\begin{proof}[\pfref{prop:injective_ub}]
Equation \pref{eq:injective_ub_pseudo} follows by rescaling a given
pseudodistribution $\mu$ by using
$x'=x/\sqrt{\Ep_{\mu}\nrm*{x}_{2}^{2}}$ and so forth, so that the
pseudodistribution is feasible for the maximization problem \pref{eq:sos_injective}.

For \pref{eq:injective_ub_sos}, let $\mu$ be a degree-$k$
pseudodistribution with $\mu\entails\crl*{\nrm*{z}_{2}^{2}\leq{}1}$. Then,
using \pref{eq:injective_ub_pseudo} and the AM-GM inequality we get
$\tri*{T,\Ep_{\mu}\brk*{x\tens{}y\tens{}z}}\leq\frac{\nrm*{T}_{\injt_k}}{2}\prn*{\Ep_{\mu}\nrm*{x}^{2}_{2}
  + \Ep_{\mu}\nrm*{y}^{2}_{2}}$. Using linearity of the pseudoexpectation
operator we have $\Ep_{\mu}\brk*{\frac{\nrm*{T}_{\injt_k}}{2}\prn*{\nrm*{x}^{2}_{2}
    + \nrm*{y}^{2}_{2}}-\tri*{T,x\tens{}y\tens{}z}}\geq{}0$, and so \pref{eq:injective_ub_sos} follows from the
duality of pseudoexpectations and sum-of-squares proofs. The remaining
statements follow by symmetry.
\end{proof}

%\subsection{Flattenings}
\subsubsection{Projections}
Here we state some basic results regarding the projection operators defined in \pref{sec:subgradient_norm}.
\begin{proposition}
\label{prop:qperp_kernel}
Let $x,y,z\in\bbR^{d}$ be given. If  at least two of the follow
conditions hold:
\[1)\,\, x\in\bbU,\quad 2)\,\, y\in\bbV,\quad 3)\,\, z\in\bbW,\] then
$\cQ_{T^{\perp}}(x\tens{}y\tens{}z)=0$.
\end{proposition}
\begin{proof}[\pfref{prop:qperp_kernel}]
Suppose that $x\in\bbU$ and $y\in\bbV$. Then
$\cP_{\bbU^{\perp}}(x)=0$, and so
$\cQ_{T^{\perp}}^{0}(x\tens{}y\tens{}z)=\cQ_{T^{\perp}}^{2}(x\tens{}y\tens{}z)=\cQ_{T^{\perp}}^{3}(x\tens{}y\tens{}z)=0$. We
also have $\cP_{\bbV^{\perp}}(y)=0$, and so
$\cQ_{T^{\perp}}^{1}(x\tens{}y\tens{}z)=0$. The remaining cases follow by symmetry.
\end{proof}

\begin{lemma}
\label{lem:projection_injective}
Let $k\geq{}4$. For any tensor $X\in\RTthree$, and any subspaces $\bbU,\bbV,\bbW$ we have
\[
\nrm*{(\cP_{\bbU}\tens{}\cP_{\bbV}\tens{}\cP_{\bbW})X}_{\injt_k}\leq{}
\nrm*{X}_{\injt_k},
\]
and in particular $\nrm*{\QTpara(X)}_{\injt_k} \leq{}
4\nrm*{X}_{\injt_k}$ and $\nrm*{\QTperp(X)}_{\injt_k} \leq{} 4\nrm*{X}_{\injt_k}$.
\end{lemma}
\begin{proof}[\pfref{lem:projection_injective}]
For any degree-$4$ pseudodistribution $\mu$ over indeterminates
$x,y,z$ we have
\begin{align*}
\Ep_{\mu}\tri*{(\cP_{\bbU}\tens{}\cP_{\bbV}\tens{}\cP_{\bbW})X,x\tens{}y\tens{}z}
&=
\Ep_{\mu}\tri*{X,(\cP_{\bbU}x)\tens{}(\cP_{\bbV}y)\tens{}(\cP_{\bbW}z)} \\
&\leq{}
\sqrt{
\Ep_{\mu}\nrm*{\cP_{\bbU}x}_{2}^{2}\cdot\Ep_{\mu}\nrm*{\cP_{\bbV}y}_{2}^{2}\cdot\Ep_{\mu}\nrm*{\cP_{\bbW}z}_{2}^{2}
} \\
&\leq{}
\sqrt{
\Ep_{\mu}\nrm*{x}_{2}^{2}\cdot\Ep_{\mu}\nrm*{y}_{2}^{2}\cdot\Ep_{\mu}\nrm*{z}_{2}^{2}
},
\end{align*}
where the first inequality uses \pref{prop:injective_ub} and the
second uses that
$\provable_{2}\nrm*{\cP_{\bbX}x}_{2}^{2}\leq{}\nrm*{x}_{2}^{2}$ for
any subspace $\bbX$. This establishes the first result.

Now observe from \pref{eq:all_subspaces} that $\nrm*{\QTpara(T)}_{\injt_k} \leq{}
\sum_{i=1}^{4}\nrm{\cQ_{T^{\para}}^{i}(T)}_{\injt_k}$ and
$\nrm*{\QTperp(T)}_{\injt_k} \leq{}
\sum_{i=1}^{4}\nrm{\cQ_{T^{\perp}}^{i}(T)}_{\injt_k}$. We thus obtain the
second result by applying the first to each of the summands.
\end{proof}

\subsubsection{Flattenings}
The multilinear rank of a tensor $T\in\RTthree$ is the triple
$(r_1,r_2,r_3)$, where
\begin{equation}
\label{eq:multilinear}
r_1(T) = \dim \mathrm{span}\crl*{T_{\cdot,j,k}
\mid{}j,k\in\brk*{d}},
\end{equation}
is the dimension of the space spanned by the mode-1 fibers and
$r_2(T)$ and $r_3(T)$ are defined likewise for the second and third mode.

We define the $i$th flattening map
$\flatten_{i}:\RTthree\to\bbR^{d\times{}d^{2}}$ via
\begin{equation}
\flatten_{1}(T)_{i,(j,k)} = \flatten_{2}(T)_{j,(i,k)} =
\flatten_{3}(T)_{k,(i,j)} = T_{i,j,k}.\label{eq:flatten}
\end{equation}
A standard result is that
$\rank(\flat_i(T))  = r_i(T)$ \citep{friedland2018nuclear}. We also
have the following comparison between the nuclear norm of the tensor
and its flattenings.
\begin{lemma}[\cite{friedland2018nuclear}, Theorem 9.4]
\label{lem:flattening_rank}
For any tensor $T\in\RTthree$,
\[
\nrm*{T}_{\nuc} \leq{}
\min\crl*{
  \sqrt{\min\crl*{r_2,r_3}}\nrm*{\flatten_1(T)}_{\nuc},
  \sqrt{\min\crl*{r_1,r_3}}\nrm*{\flatten_2(T)}_{\nuc},
\sqrt{\min\crl*{r_1,r_2}}\nrm*{\flatten_3(T)}_{\nuc}
}.
\]
\end{lemma}

\subsubsection{Concentration}
\begin{lemma}[Talagrand-type concentration for supremum of empirical process]
\label{lem:talagrand}
Let $\cF$ be a class of functions of the form $f:\cZ\to\bbR$. Let
$z_{1},\ldots,z_n$ be sampled i.i.d. from a distribution $\cD$ over
$\cZ$ that satisfies $\En\brk*{f(z)}=0$ and has $\abs*{f(z)}\leq{}c$
almost surely. Let $\sigma^{2}=\sup_{f\in\cF}\En{}f^{2}(z)$. Then for
any $\delta>0$, with probability at least $1-2\delta$ over the i.i.d. draw of $z_1,\ldots,z_n$,
\[
\sup_{f\in\cF}\abs*{\frac{1}{n}\sum_{t=1}^{n}f(z_t)}
\leq{} 4\En_{z_{1:n}}\En_{\eps}\sup_{f\in\cF}\frac{1}{n}\sum_{t=1}^{n}\eps_{t}f(z_t)
+ \sqrt{
\frac{2\sigma^{2}\log(1/\delta)}{n}
}
+ \frac{2c\log(1/\delta)}{n}.
\]
\end{lemma}
\begin{proof}[\pfref{lem:talagrand}]
Follows from Theorem A.1 of \cite{bartlett2005local} applied to the classes $\cF$ and
$-\cF$ separately, along with the standard in-expectation
symmetrization lemma for uniform convergence.
\end{proof}

%%% Local Variables:
%%% mode: latex
%%% TeX-master: "paper"
%%% End:

\section{Proofs from \pref{sec:subgradient}}
\label{app:subgradient}
% !TEX root = paper.tex

The main result in this section is to prove
\pref{thm:sos_subgradient}, then use this result to prove
\pref{thm:sos_norm_comparison}. Before proceeding to the main proofs we state an intermediate result.
\begin{lemma}[\cite{potechin2017exact}]
\label{lem:trilinear_ps}
Let $x,y,z\in\bbR^{d}$ be indeterminates. Let
$\cA=\crl*{\nrm*{y}_{2}^{2}=1}$. Then for any $r\in\brk*{d}$,
\[
\cA\provable_{6} \sum_{i=1}^{r}x_iy_iz_i \leq{}
\frac{1}{2}\nrm*{x}_{2}^{2} + \frac{1}{2}\nrm*{z}_{2}^{2}
-\frac{1}{4}\sum_{i=r+1}^{d}x_i^{2} + z_i^{2}
-\frac{1}{8}\sum_{i=1}^{d}\sum_{j\neq{}i}y_i^{2}\prn*{x_j^{2} +
z_{j}^{2} + y_j^{2}(\nrm*{x}_{2}^{2} + \nrm*{z}_{2}^{2})}.
\]
\end{lemma}
\begin{corollary}
\label{cor:trilinear_ps}
Let $\cA=\crl*{\nrm*{x}_{2}^{2}=1, \nrm*{z}_{2}^{2}=1}$. Then for any $r\in\brk*{d}$,
\[
\cA\provable_{6} \sum_{i=1}^{r}x_iy_iz_i \leq{}
1
-\frac{1}{4}\sum_{i=r+1}^{d}x_i^{2} + z_i^{2}
-\frac{1}{8}\sum_{i=1}^{d}\sum_{j\neq{}i}y_i^{2}\prn*{x_j^{2} +
y_j^{2} + z_{j}^{2}}.
\]
\end{corollary}
\begin{proof}[\pfref{cor:trilinear_ps}]
We will show that 
$\cA\provable_{6}-\sum_{i=1}^{d}\sum_{j\neq{}i}y_{i}^{2}y_{j}^{2}\nrm*{x}_{2}^{2}=-\sum_{i=1}^{d}\sum_{j\neq{}i}y_{i}^{2}y_{j}^{2}$;
the term involving $\nrm*{z}_{2}^{2}$ follows from the same
reasoning. The desired inequality is equivalent to
$\sum_{i=1}^{d}\sum_{j\neq{}i}y_{i}^{2}y_{j}^{2}(\nrm*{x}_{2}^{2}-1)=0$,
which is clearly the product of a degree-$4$ polynomial and the
equality constraint $\crl*{\nrm*{x}_{2}^{2}-1=0}$.
\end{proof}

\begin{proof}[\pfref{thm:sos_subgradient}]
% ~\\
% \emph{\savehyperref{thm:sos_subgradient}{Part 1}.}
\emph{Preliminaries.} We first claim that the following equalities hold:
\begin{itemize}
\item $\QTpara^{0}(X^{\star}) = X^{\star}$.
\item
  $\nrm*{X^{\star}}_{\injt_k}=\nrm*{X^{\star}}_{\injt_{k+2}}=\nrm*{X^{\star}}_{\inj}=1$.
\item
  $\tri*{X^{\star},T}=\nrm*{T}_{\nuct_k}=\nrm*{T}_{\nuct_{k+2}}=\nrm*{T}_{\nuc}$.
\end{itemize}
% Define $X^{\star}=\sum_{i=1}^{r}u_i\tens{}v_i\tens{}w_i$. 
Indeed, it is immediate from the definition of $X^{\star}$ that $\QTpara^{0}(X^{\star})=X^{\star}$, and
it follows from \pref{lem:nuclear_orthogonal} that $\nrm*{X^{\star}}_{\injt_k}=1$ and
$\tri*{X^{\star},T}=\nrm*{T}_{\nuct_k}$ for all
$k\geq{}4$. % We now move on to the core of the proof.

% \noindent\emph{\savehyperref{thm:sos_subgradient}{Part 2}. Bounding
%   dual norm is sufficient.}
\noindent\emph{Bounding dual norm is sufficient.}
To establish the inequality \pref{eq:sos_subgradient}, we reduce to a simpler problem. The claim is as follows: Fix a constant $\alpha>0$. If for all
$X\in\RTthree$ with $\nrm*{X}_{\injt_k}\leq{}\alpha$, we
have 
\begin{equation}
  \label{eq:subgradient_dual}
  \nrm*{X^{\star} + \QTperp(X)}_{\injt_{k+2}}\leq{}1,
\end{equation}
then \pref{eq:sos_subgradient} holds for all $X$ with $\nrm*{X}_{\injt_{k}}\leq{}\alpha$.
To see that this is the case, observe that for any $T'$ and all such $X$ we have
\begin{align*}
  \tri*{X^{\star} + \QTperp(X),T'-T} &\leq{} \nrm*{X^{\star} +
                                       \QTperp(X)}_{\injt_{k+2}}\nrm*{T'}_{\nuct_{k+2} }
                                       - \tri*{X^{\star} + \QTperp(X),T} \\
                                     &\leq{} \nrm*{T'}_{\nuct_{k+2}}
                                       - \tri*{X^{\star} + \QTperp(X),T} \\
                                     &= \nrm*{T'}_{\nuct_{k+2}}
                                       - \nrm*{T}_{\nuct_{k+2}},
\end{align*}
where the first inequality uses \pref{prop:sos_dual}, the second
inequality uses \pref{eq:subgradient_dual}, and the final equality
uses the definition of $X^{\star}$ and that
$\tri*{\QTperp(X),T}=\tri*{X,\QTperp(T)}=0$. Rearranging the
inequality yields \pref{eq:sos_subgradient}.

\noindent\emph{Bounding the dual norm.} The remainder of the proof
establishes that \pref{eq:subgradient_dual} holds for $\alpha=1/64$.

Let $x,y,z\in\bbR^{d}$ be indeterminates. We will provide a degree-$(k+2)$
SoS upper bound on the polynomial
$\tri*{X^{\star}+\QTperp(X),x\tens{}y\tens{}z}$, which will
suffice to establish \pref{eq:subgradient_dual}.

Let $\crl*{u_i}_{i=1}^{r}$,
$\crl*{v_i}_{i=1}^{r}$, $\crl*{w_i}_{i=1}^{r}$ be as in
\pref{eq:orthogonal}. Then, let $\bbU=\mathrm{span}(\crl*{u_i}_{i=1}^{r})$, $\bbV=\mathrm{span}(\crl*{v_i}_{i=1}^{r})$, and 
$\bbW=\mathrm{span}(\crl*{w_i}_{i=1}^{r})$, so that $\crl*{u_i}_{i=1}^{r}$ is a basis for $\bbU$, and
likewise for the other modes. Let $\crl*{u_i}_{i=r+1}^{d}$ be an arbitrary
orthonormal basis for $\bbU^{\perp}$ and likewise with
$\crl*{v_i}_{i=r+1}^{d}$ for $\bbV^{\perp}$, and
$\crl*{w_i}_{i=r+1}^{d}$ for $\bbW^{\perp}$.

We perform a change of basis and let $u_i=v_i=w_i=e_i$, where
$e_i$ is the $i$th standard basis vector. Then with $x,y,z$ expressed in the new basis we can write
\begin{equation}
\tri*{X^{\star}+\QTperp(X),x\tens{}y\tens{}z} =
\sum_{i=1}^{r}x_iy_iz_i + 
\tri*{\QTperp(X),x\tens{}y\tens{}z}.\label{eq:dual_split}
\end{equation}
Let $\cA=\crl*{\nrm*{x}_{2}^{2}=1, \nrm*{y}_{2}^{2}=1,
   \nrm*{z}_{2}^{2}=1}$. From \pref{cor:trilinear_ps}, we have
\[
\cA\provable_{6} \sum_{i=1}^{r}x_iy_iz_i \leq{}
\frac{1}{2}\nrm*{x}_{2}^{2} + \frac{1}{2}\nrm*{z}_{2}^{2}
-\frac{1}{4}\sum_{i=r+1}^{d}x_i^{2} + z_i^{2}
-\frac{1}{8}\sum_{i=1}^{d}\sum_{j\neq{}i}y_i^{2}\prn*{x_j^{2} +
z_{j}^{2} + y_j^{2}}.
\]
We now handle the second term in \pref{eq:dual_split}. We will
establish that
\begin{equation}
\label{eq:subspace_ps}
\cA\provable_{k+2}\tri*{\QTperp(X),x\tens{}y\tens{}z} \leq{}
+O(\alpha)\cdot\sum_{i=r+1}^{d}x_i^{2} + z_i^{2}
+O(\alpha)\cdot\sum_{i=1}^{d}\sum_{j\neq{}i}y_i^{2}\prn*{x_j^{2} +
z_{j}^{2} + y_j^{2}}.
\end{equation}
under the assumption that $\nrm*{X}_{\injt_k}\leq{}\alpha$. To do this it
suffices show that for each $i$ individually, 
\[
\cA\provable_{6}y_{i}^{2}\tri*{\QTperp(X),x\tens{}y\tens{}z} \leq{}
O(\alpha)\cdot{}y_{i}^{2}\sum_{j\neq{}i}\prn*{x_j^{2} +
z_{j}^{2} + y_j^{2}},
\]
with an extra additive factor of
$O(\alpha)\cdot{}y_{i}^{2}(x_{i}^{2}+z_{i}^{2})$ when $i>r$.

Let $1\leq{}i\leq{}d$ be fixed and let $x'=x-x_ie_i$,
$y'=y-y_ie_i$, and $z'=z-z_ie_i$. We write

\begin{equation}
\label{eq:qt_multilinear}
\tri*{\QTperp(X),x\tens{}y\tens{}z}
= \tri*{X,\QTperp((x_ie_i +
  x')\tens{}(y_ie_i+y')\tens{}(z_ie_i+z'))}.
\end{equation}

\emph{Case: $i\leq{}r$.}
Observe that with our change of basis we have $e_i\in\bbU$ along the
first mode, $e_i\in\bbV$ along the second mode, and $e_i\in\bbW$ along
the third mode. In view of \pref{prop:qperp_kernel}, this means we
have
\begin{align*}
0&=\QTperp((x_ie_i)\tens{}(y_ie_i)\tens{}(z_ie_i))\\
&= 
\QTperp((x_ie_i)\tens{}(y_ie_i)\tens{}z') \\
&= \QTperp(x'\tens{}(y_ie_i)\tens{}(z_ie_i))\\
&= \QTperp((x_ie_i)\tens{}y'\tens{}(z_ie_i)).
\end{align*}
Consequently, using multilinearity we can write
\begin{equation}
\begin{aligned}
\tri*{\QTperp(X),x\tens{}y\tens{}z}
&= \tri*{\QTperp(X),x'\tens{}y'\tens{}z'}
+ \tri*{\QTperp(X),x'\tens{}y'\tens{}(z_ie_i)}\\
&~~~~+ \tri*{\QTperp(X),(x_ie_i)\tens{}y'\tens{}z'}
+ \tri*{\QTperp(X),x'\tens{}(y_ie_i)\tens{}z'}.
\end{aligned}\label{eq:qt_zeros}
\end{equation}

In what follows we will repeatedly invoke that 
$\cA\provable_{2}\nrm*{x_ie_i}_{2}^{2}\leq{}\nrm*{x}_{2}^{2}\leq{}1$,
$\cA\provable_{2}\nrm*{x'}_{2}^{2}\leq{}\nrm*{x}_{2}^{2}\leq{}1$, and
so forth. We will also use that if $\nrm*{X}_{\injt_k}\leq\alpha$ then---via \pref{prop:injective_ub} and
\pref{lem:projection_injective}---for any indeterminates $a,b,c$,
\begin{align*}
&\crl*{\nrm*{a}_{2}^{2}\leq{}1}
  \provable_{k}\tri*{\cQ_{T^{\perp}}(X),a\tens{}b\tens{}c}\leq{}
  2\alpha\prn*{\nrm*{b}_{2}^{2}+\nrm*{c}_{2}^{2}},\\
&\crl*{\nrm*{b}_{2}^{2}\leq{}1}
  \provable_{k}\tri*{\cQ_{T^{\perp}}(X),a\tens{}b\tens{}c}\leq{}
  2\alpha\prn*{\nrm*{a}_{2}^{2}+\nrm*{c}_{2}^{2}},\\
&\crl*{\nrm*{c}_{2}^{2}\leq{}1}
  \provable_{k}\tri*{\cQ_{T^{\perp}}(X),a\tens{}b\tens{}c}\leq{} 2\alpha\prn*{\nrm*{b}_{2}^{2}+\nrm*{b}_{2}^{2}}.
\end{align*}
These inequalities allow us to bound the terms in \pref{eq:qt_zeros} as follows:
\begin{align*}
  &\crl*{\nrm*{y}_{2}^{2}\leq1}\provable_{k}
  \tri*{\QTperp(X),x'\tens{}y'\tens{}z'}\leq{}
    2\alpha\prn*{\nrm*{x'}_{2}^{2}+\nrm*{z'}_{2}^{2}}.\\
  &\crl*{\nrm*{z}_{2}^{2}\leq1}\provable_{k}
  \tri*{\QTperp(X),x'\tens{}y'\tens{}(z_ie_i)}\leq{}
    2\alpha\prn*{\nrm*{x'}_{2}^{2}+\nrm*{y'}_{2}^{2}}.\\
  &\crl*{\nrm*{x}_{2}^{2}\leq1}\provable_{k}
  \tri*{\QTperp(X),(x_ie_i)\tens{}y'\tens{}z'}\leq{}
    2\alpha\prn*{\nrm*{y'}_{2}^{2}+\nrm*{z'}_{2}^{2}}.\\
  &\crl*{\nrm*{y}_{2}^{2}\leq
1}\provable_{k}
  \tri*{\QTperp(X),x'\tens{}(y_ie_i)\tens{}z'}\leq{} 2\alpha\prn*{\nrm*{x'}_{2}^{2}+\nrm*{z'}_{2}^{2}}.
\end{align*}
Adding these inequalities, we get
\[
\cA\provable_{k}\tri*{\QTperp(X),x\tens{}y\tens{}z} \leq{} 6\alpha\sum_{j\neq{}i}(x_{j}^{2}+y_{j}^{2}+z_{j}^{2}),
\]
and by the multiplication rule for SoS proofs,
\[
\cA\provable_{k+2}y_{i}^{2}\tri*{\QTperp(X),x\tens{}y\tens{}z}\leq 6\alpha\sum_{j\neq{}i}y_{i}^{2}(x_{j}^{2}+y_{j}^{2}+z_{j}^{2}).
\]
\emph{Case: $i>r$.} As in the previous case, we split the expression $\tri*{\QTperp(X),(x_ie_i +
  x')\tens{}(y_ie_i+y')\tens{}(z_ie_i+z')}$ in \pref{eq:qt_multilinear} using multilinearity, however
we can no longer argue that four of the eight terms vanish. As a
starting point, for the four terms that appeared in the $i\leq{}r$ case we can adopt the same upper bound to get
\begin{align*}
  \cA\provable_k\tri*{\QTperp(X),x\tens{}y\tens{}z}
  &\leq \QTperp((x_ie_i)\tens{}(y_ie_i)\tens{}(z_ie_i)) + \QTperp((x_ie_i)\tens{}y'\tens{}(z_ie_i)) \\
  &~~~~ +  \QTperp((x_ie_i)\tens{}(y_ie_i)\tens{}z') 
+ \QTperp(x'\tens{}(y_ie_i)\tens{}(z_ie_i))\\
&~~~~+ 6\alpha\sum_{j\neq{}i}(x_{j}^{2}+y_{j}^{2}+z_{j}^{2}).
\end{align*}
We bound the four remaining terms as follows:
\begin{align*}
  &\crl*{\nrm*{y}_{2}^{2}\leq{}1}\provable_{k}\QTperp((x_ie_i)\tens{}(y_ie_i)\tens{}(z_ie_i))
\leq{} 2\alpha(x_{i}^{2} + z_{i}^{2}), \\
  &\crl*{\nrm*{y}_{2}^{2}\leq{}1}\provable_{k}\QTperp((x_ie_i)\tens{}(y_ie_i)\tens{}z')\leq{}2\alpha(x_{i}^{2}
  + \nrm*{z'}_{2}^{2}),\\
&\crl*{\nrm*{y}_{2}^{2}\leq{}1}\provable_{k}\QTperp(x'\tens{}(y_ie_i)\tens{}(z_ie_i))\leq{}2\alpha(\nrm*{x'}_{2}^{2}
  + z_{i}^{2}),\\
&\crl*{\nrm*{y}_{2}^{2}\leq{}1}\provable_{k}\QTperp((x_ie_i)\tens{}y'\tens{}(z_ie_i))\leq{}2\alpha(x_{i}^{2}+z_{i}^{2}).
\end{align*}
Adding together all of these inequalities, we get
\[
  \cA\provable_k\tri*{\QTperp(X),x\tens{}y\tens{}z}
  \leq 6\alpha(x_{i}^{2}+z_{i}^{2}) + 8\alpha\sum_{j\neq{}i}(x_{j}^{2}+y_{j}^{2}+z_{j}^{2}),
\]
and 
\[
  \cA\provable_{k+2} y_{i}^{2}\tri*{\QTperp(X),x\tens{}y\tens{}z}
  \leq 6\alpha{}y_{i}^{2}(x_{i}^{2}+z_{i}^{2}) +
  8\alpha\sum_{j\neq{}i}y_{i}^{2}(x_{j}^{2}+y_{j}^{2}+z_{j}^{2}).
\]

\emph{Putting everything together.}
Taking the inequalities we proved for the individual coordinates $i$
and summing them up, we have
\[
\cA\provable_{k+2}\nrm*{y}_{2}^{2}\cdot\tri*{\QTperp(X),x\tens{}y\tens{}z}
\leq{} 6\alpha\sum_{i=r+1}^{d}y_{i}^{2}(x_{i}^{2}+z_{i}^{2}) + 
8\alpha\sum_{i=1}^{d}\sum_{j\neq{}i}y_{i}^{2} (x_{j}^{2}+y_{j}^{2}+z_{j}^{2}).
\]
Since $\crl*{\nrm*{y}_{2}^{2}=1}\subset\cA$, we get 
\[
\cA\provable_{k+2}\tri*{\QTperp(X),x\tens{}y\tens{}z}
\leq{} 6\alpha\sum_{i=r+1}^{d}(x_{i}^{2}+z_{i}^{2}) + 
8\alpha\sum_{i=1}^{d}\sum_{j\neq{}i}y_{i}^{2} (x_{j}^{2}+y_{j}^{2}+z_{j}^{2}).
\]
Returning to \pref{eq:dual_split}, this inequality plus the earlier
bound from \pref{cor:trilinear_ps} imply
\begin{align*}
\cA\provable_{k+2}\,\,&\tri*{X^{\star}+\QTperp(X),x\tens{}y\tens{}z} \\
&\leq{} 
1 + 
(6\alpha-1/4)\sum_{i=r+1}^{d}\prn*{x_i^{2} + z_i^{2}}
+(8\alpha-1/8)\sum_{i=1}^{d}\sum_{j\neq{}i}y_i^{2}\prn*{x_j^{2} +
z_{j}^{2} + y_j^{2}}.\\
&\leq{} 1,\quad\text{for $\alpha\leq{}1/64$.}
\end{align*}
By the duality of SoS proofs and pseudodistributions, we have
$\nrm*{X^{\star}+\QTperp(X)}_{\injt_{k+2}}\leq{}1$ as desired.
\end{proof}

\begin{proof}[\pfref{thm:sos_norm_comparison}]
We first establish equation \pref{eq:sos_cone}. We combine the assumption that
$\nrm*{T'}_{\nuct_{k+2}}\leq{}\nrm*{T}_{\nuct_{k}}$ with equation \pref{eq:sos_subgradient} to get
\[
\nrm*{T}_{\nuct_{k+2}} \geq \nrm*{T'}_{\nuct_{k+2}} \geq{}  \nrm*{T}_{\nuct_{k+2}} + \tri*{X^{\star} + \QTperp(X),\Delta}.
\]
for all $X$ with $\nrm*{X}_{\injt_{k}}\leq{}1/64$ and $X^{\star}$ as
in \pref{thm:sos_subgradient}.
Rearranging, this yields
\[
\tri*{X,\QTperp(\Delta)}=\tri*{\QTperp(X),\Delta} \leq{}
-\tri*{X^{\star},\Delta} = -\tri*{X^{\star}, \QTpara^{0}(\Delta)}.
\]
We now use that $\nrm*{X^{\star}}_{\injt_{k}}\leq{}1$ (from
\pref{thm:sos_subgradient}) and choose $X$ to be a point obtaining the supremum in
$\nrm*{\QTperp(\Delta)}_{\nuct_{k}}=\sup_{\nrm*{X}_{\injt_k}\leq{}1}\tri*{X,\QTperp(\Delta)}$,
scaled by $1/64$. Then the inequality above implies
\[
\tri*{X,\QTperp(\Delta)} = \frac{1}{64}\cdot\nrm*{\QTperp(\Delta)}_{\nuct_k} \leq{}
\nrm*{\QTpara^{0}(\Delta)}_{\nuct_{k}}.
\]
We now establish equation \pref{eq:sos_to_frobenius}. Observe that we can write
\[
\nrm*{\Delta}_{\nuct_k}
= \nrm*{\QTpara (\Delta) + \QTperp(\Delta)}_{\nuct_k}
\leq \nrm*{\QTpara (\Delta)}_{\nuct_k} + \nrm*{\QTperp(\Delta)}_{\nuct_k}.
\]
Combining this with \pref{eq:sos_cone}, we get
\[
\nrm*{\Delta}_{\nuct_k}
\leq \nrm*{\QTpara (\Delta)}_{\nuct_k} +
64\nrm*{\QTpara^{0}(\Delta)}_{\nuct_k} 
\]
Using the triangle inequality, we upper bound the first term as
\[
\nrm*{\QTpara (\Delta)}_{\nuct_k}\leq{}\sum_{i=1}^{4}\nrm*{\QTpara^{i} (\Delta)}_{\nuct_k}\leq{}\sum_{i=1}^{4}\nrm*{\QTpara^{i} (\Delta)}_{\nuc}.
\]
To proceed, we flatten each tensor in the summation above into a matrix
and use \pref{lem:flattening_rank} to show that the nuclear norm of
the flattening leads to an upper bound. Let $r_1=\dim{}\bbU$, $r_2=\dim\bbV$, and
$r_3=\dim\bbW$. Then the following inequalities hold
\begin{align*}
&\nrm*{\QTpara^{1}
  (\Delta)}_{\nuc}\leq{}\sqrt{r_{2}}\nrm*{\flatten_{3}(\QTpara^{1}
  (\Delta))}_{\nuc}
\leq{}\sqrt{r_{2}r_{3}}\nrm*{\QTpara^{1} (\Delta)}_{\frobnorm}.\\
&\nrm*{\QTpara^{2}
  (\Delta)}_{\nuc}\leq{}\sqrt{r_{3}}\nrm*{\flatten_{1}(\QTpara^{2}
  (\Delta))}_{\nuc}
\leq{}\sqrt{r_{1}r_{3}}\nrm*{\QTpara^{2} (\Delta)}_{\frobnorm}.\\
&\nrm*{\QTpara^{3}
  (\Delta)}_{\nuc}\leq{}\sqrt{r_{1}}\nrm*{\flatten_{2}(\QTpara^{3}
  (\Delta))}_{\nuc}
\leq{}\sqrt{r_{1}r_{2}}\nrm*{\QTpara^{3} (\Delta)}_{\frobnorm}.
\end{align*}
The first inequality in each line above follows from
\pref{lem:flattening_rank} and the definitions in
\pref{eq:all_subspaces}. The second follows from the fact that
$\nrm*{A}_{\nuc}\leq{}\sqrt{\rank(A)}\nrm*{A}_{\frobnorm}$ for any matrix $A$, along with
the fact that $\rank(\flatten_{i}(T))=r_i(T)$ for any tensor, and that
flattening does not change the entrywise $\ls_{2}$ norm. We have $\nrm*{\QTpara^{0}
  (\Delta)}_{\nuc}\leq{}\sqrt{r_{2}r_{3}}\nrm*{\QTpara^{0}
  (\Delta)}_{\frobnorm}$ as well by the same argument, though the choice of
$r_1/r_2/r_3$ in this case is arbitrary.

To combine all the bounds, we use that $r_1,r_2,r_3\leq{}r$ and that orthogonal projection
decreases the $\ls_{2}$ norm, which yields
\[
\nrm*{\Delta}_{\nuct_k} \leq{} 68r\nrm*{\Delta}_{\frobnorm}.
\]

\end{proof}

%%% Local Variables:
%%% mode: latex
%%% TeX-master: "paper"
%%% End:

\section{Proofs from \pref{sec:main_results}}
\label{app:main_results}
% !TEX root = paper.tex

This section of the appendix is structured as follows. 

First, in
\pref{app:generalization}, we provide we provide a generalization bound
for general classes of tensors and measurement models, from which all of our main statistical
results will follow as special cases. This bound assumes that a restricted
eigenvalue-type property holds for the tensor class and
measurement model under consideration. 

In \pref{app:re_tensor_completion} and
\pref{app:re_tensor_sensing} we establish that this restricted
eigenvalue property holds for the measurement models in \pref{sec:main_results}.

In \pref{app:main_theorems} we combine these results to prove the
main results of that section.

\subsection{Agnostic generalization bounds for tensor classes}
\label{app:generalization}

In this section we given generalization guarantees for empirical risk
minimization in a general learning setup. We receive a set of examples
$S\ldef(X_1,Y_1),\ldots,(X_n,Y_n)$ i.i.d. from a distribution $\cD$
over $\RTthree\times{}\bbR$. We assume that a convex class of tensors
e$\cT\subseteq\RTthree$ is given, and that our goal is to achieve excess
risk against an unknown benchmark $T^{\star}\in\cT$.  We analyze the performance of empirical risk minimization over $\cT$: 
\[
\term=\argmin_{T\in\cT}\emprisk(T),\] where $\emprisk$ is the
empirical square loss. The main result from this section is
\pref{thm:generic_offset}, which bounds the performance of ERM under various assumptions on the data distribution, the class $\cT$, and the benchmark $T^{\star}$. To state the result, recall from \pref{sec:main_results} that $\Sigma$ is the population correlation matrix and $\dataop$ is the empirical design operator.
\begin{theorem}
  \label{thm:generic_offset}
Let a benchmark $T^{\star}\in\cT$ be fixed, and let $\xi_t=\prn*{Y_t-\tri*{T^{\star},X_t}}$. Suppose there exist a pair of dual norms $\nrm*{\cdot}$ and $\nrm*{\cdot}_{\star}$ for which the following conditions hold:
\begin{enumerate}
\item There are constants $0\leq{}c<2$, $\gamma_{n}>0$, and $\delta_0\geq0$ such that
  with probability at least $1-\delta_{0}$, 
\[
\nrm*{\Sigma^{1/2}\Delta}^{\frobnorm}_{2}\leq{}
\frac{c}{n}\nrm*{\dataop(\Delta)}_{\frobnorm}^{2} + \gamma_{n}\quad\forall{}\Delta\in\cT-T^{\star}\quad\text{(\propone)}.
\]
\item There are constants $M\geq{}0$ and $\delta_1\geq{}0$ such that with
  probability at least
  $1-\delta_1$, 
\[
\nrm*{\sum_{t=1}^{n}\xi_tX_t-\En\brk*{\xi{}X}}\leq{}M\cdot{}\sqrt{n} \quad\text{(\proptwo)}.
\]
\item There is a constant $\kappa\geq{}0$ such that
  \[
  \nrm*{\Delta}_{\star}^{2}\leq{}\kappa^{2}\cdot\nrm*{\Sigma^{1/2}\Delta}_{\frobnorm}^{2}\quad\forall\Delta\in\cT-T^{\star}\quad\text{(\propthree)}.
\] 
% for
  % all $\Delta\in\cT-T^{\star}$.
\end{enumerate}
Then with probability at least
$1-(\delta_0+\delta_1)$,
\begin{align}
  \label{eq:tensor_risk}
  \poprisk(\term) - \poprisk(T^{\star}) &\leq
\frac{2\kappa^{2}M^{2}}{c'n} + 2\gamma_n.
\end{align}

\end{theorem}

\begin{proof}[\pfref{thm:generic_offset}]
Since $\cT$ is convex, and since $\term$ minimizes the (strongly
convex) empirical risk, we have
\[
\emprisk(T^{\star}) - \emprisk(\term) \geq{}
\tri*{\grad\emprisk(\term),T^{\star}-\term} +
\Eh\tri*{\term-T^{\star},X}^{2} \geq{} \Eh\tri*{\term-T^{\star},X}^{2}.
\]
By rearranging and expanding the definition of $\emprisk$, this implies
\[
\Eh\prn*{\tri*{T^{\star},X}-Y}^{2}-\Eh\prn*{\tri{\term,X}-Y}^{2}-
\Eh\tri*{\term-T^{\star},X}^{2}\geq{}0.\] 
Since the left-hand side is non-negative, we can add it to the
population excess risk, which implies
\begin{align*}
&\En\prn*{\tri{\term,X}-Y}^{2} - \En\prn*{\tri*{T^{\star},X}-Y}^{2} \\
&\leq
\En\prn*{\tri{\term,X}-Y}^{2} - \En\prn*{\tri*{T^{\star},X}-Y}^{2} + 
\Eh\prn*{\tri*{T^{\star},X}-Y}^{2}-\Eh\prn*{\tri{\term,X}-Y}^{2}-
\Eh\tri*{\term-T^{\star},X}^{2}.
\intertext{Rearranging, this is equal to}
&= (\En-\Eh)\brk*{2\tri*{\term-\ts,X}(\tri*{\ts,X}-Y)} +
  \En\tri*{\term-T^{\star},X}^{2} - 2\Eh\tri*{\term-T^{\star},X}^{2}.
\intertext{Since $\term-\ts$ is an element of $\cT-\ts$, we move
  to an upper bound by taking a supremum over elements of this set.}
&\leq
\sup_{\Delta\in\cT-T^{\star}}\crl*{ (\En-\Eh)\brk*{2\tri*{\Delta,X}(\tri*{\ts,X}-Y)} +
  \En\tri*{\Delta,X}^{2} - 2\Eh\tri*{\Delta,X}^{2}}.
\end{align*}
This establishes inequality \pref{eq:offset_process}. It remains to
use the assumptions in the theorem statement to bound the
process. \propone{} states that with probability
at least $1-\delta_0$,
$\En\tri*{\Delta,X}^{2}\leq{}c\cdot{}\Eh\tri*{\Delta,X}^{2} +
\gamma_n$ for all $\Delta\in\cT-T^{\star}$. Define $c'=2-c>0$, so that
conditioned on this event we have
\begin{align*}
  &\sup_{\Delta\in\cT-T^{\star}}\crl*{ (\En-\Eh)\brk*{2\tri*{\Delta,X}(\tri*{\ts,X}-Y)} +
  \En\tri*{\Delta,X}^{2} - 2\Eh\tri*{\Delta,X}^{2}} \\
&\leq \sup_{\Delta\in\cT-T^{\star}}\crl*{ (\En-\Eh)\brk*{2\tri*{\Delta,X}(\tri*{\ts,X}-Y)} 
  -c'\cdot\Eh\tri*{\Delta,X}^{2}} + \gamma_n.
\end{align*}
We expand the first term in the supremum as
\begin{align*}
  (\En-\Eh)\brk*{2\tri*{\Delta,X}(\tri*{\ts,X}-Y)}
&=
  -2\frac{1}{n}\sum_{t=1}^{n}\xi_t\tri*{\Delta,X_t}-\En\brk*{\xi\tri*{\Delta,X}}\\
&= -2\tri*{\frac{1}{n}\sum_{t=1}^{n}\xi_tX_t-\En\brk*{\xi{}X},\Delta}.
\end{align*}
% where $\zeta$ is an independent copy of $\tri*{\ts,X}-Y$.
It follows from \Holder's inequality that this expression is bounded as
\[
2\nrm*{\frac{1}{n}\sum_{t=1}^{n}\xi_tX_t-\En\brk*{\xi{}X}}\cdot\nrm*{\Delta}_{\star}.
\]
Defining $\psi_{n} =
\nrm*{\frac{1}{n}\sum_{t=1}^{n}\xi_tX_t-\En\brk*{\xi{}X}}$, the
development so far states that with probability at least
$1-\delta_0$,
\[
\En\prn*{\tri{\term,X}-Y}^{2} - \En\prn*{\tri*{T^{\star},X}-Y}^{2}
\leq{}
\sup_{\Delta\in\cT-T^{\star}}\crl*{ 2\psi_n\cdot\nrm*{\Delta}_{\star}
  -c'\cdot\frac{1}{n}\nrm*{\dataop(\Delta)}_{2}^{2}}
 +\gamma_n.
\]
Using \propthree we upper bound the leading term by
\begin{align*}
&\sup_{\Delta\in\cT-T^{\star}}\crl*{ 2\kappa\psi_n\cdot\nrm*{\Sigma^{1/2}\Delta}_{\frobnorm}
  -c'\cdot\frac{1}{n}\nrm*{\dataop(\Delta)}_{2}^{2}},
\intertext{and the event we already conditioned on implies that this
  is at most }
&\sup_{\Delta\in\cT-T^{\star}}\crl*{ 4\kappa\psi_n\sqrt{\frac{1}{n}\nrm*{\dataop(\Delta)}_{2}^{2}}
  -c'\cdot\frac{1}{n}\nrm*{\dataop(\Delta)}_{2}^{2}} +
  2\kappa\psi_n\sqrt{\gamma_n}\\
&\leq{}\frac{\kappa^{2}\psi_n^{2}}{c'} +  2\kappa\psi_n\sqrt{\gamma_n},
\end{align*}
where the second inequality follows from AM-GM. Finally, by
\proptwo{} we have that
with probability at least $1-\delta_1$, $\psi_n \leq{} M/\sqrt{n}$,
which leads to the final bound of
\[
\frac{\kappa^{2}M^{2}}{c'n} +  2\kappa{}M\sqrt{\frac{\gamma_n}{n}} +
\gamma_n \leq{} \frac{2\kappa^{2}M^{2}}{c'n} + 2\gamma_n.
\]

\end{proof}

\subsection{Restricted eigenvalue for tensor completion}
\label{app:re_tensor_completion}

The main result in this section is \pref{thm:tensor_completion_re}, which relates the empirical
covariance and population covariance for all tensors with bounded
entries and bounded \sos nuclear norm under sampling model for tensor
completion. This result is then used in the proof of
\pref{thm:tensor_completion} to establish a restricted eigenvalue guarantee.

\begin{theorem}
\label{thm:tensor_completion_re}
Let $d^{3/2}\leq{}n\leq{}d^{3}$, and let $\veps\in(0,1/2)$. Suppose
observed entries are drawn with replacement. Then for any $\delta>0$, with probability at least $1-\delta$,
all $T\in\RTthree$ with $\nrm*{T}_{\infty}\leq{}R$ satisfy
\[
\nrm*{\Sigma^{1/2}T}_{\frobnorm}^{2}\leq \frac{(1+2\veps)}{n}\nrm*{\dataop{}(T)}_{2}^{2} + 
O\prn*{
R\nrm*{T}_{\nuct_4}\sqrt{\frac{\log^{6}d}{nd^{3/2}}} 
+ \frac{R^{2}\log(\log^{2}(d^{3/2}\sqrt{n})/\delta)}{\veps{}n}
}.
\]
\end{theorem}
The proofs in this section pass back and forth between the
with-replacement sampling model for tensor completion used in the main body of the paper and
a without-replacement model, in which each entry is only observed a
single time for tensor completion. To distinguish between the models, we use
$S\sim\cD^{n}_{\wr}$ to refer to the draw of the dataset under the
with-replacement sampling model and $S\sim\cD^{n}_{\wo}$ to refer to
the draw under the without-replacement
model. $\Omega\subseteq\brk*{d}^{3}$ will denote the support of the
set of observed entries.

Before proving \pref{thm:tensor_completion_re}, we state and prove a
number of auxiliary lemmas, from which the main result will follow.

\begin{proposition}
\label{prop:bernstein_wo}
Let $N_{j,k}=\abs*{\crl*{t\in\brk*{n}\mid{}j_t=j,k_t=k}}$. Then
\[
\Pr_{S\sim{}\dwo}\prn*{N_{j,k}\geq{}\frac{2n}{d^{2}}
  +\frac{10 \log(1/\delta)}{3}}\leq{}\delta.
\]
\end{proposition}

\begin{proof}[\pfref{prop:bernstein_wo}]
This is an immediate consequence of Bernstein's inequality for
without-replacement sampling. See \cite{bardenet2015concentration}, Proposition 1.4.
\end{proof}

\begin{lemma}
\label{lem:tensor_completion_rademacher}
The Rademacher complexity of
$\nrm*{\cdot}_{\injt_4}$ under without-replacement sampling is bounded as
\begin{equation}
  \label{eq:tensor_completion_rademacher_wo}
  \En_{S\sim\dwo}\En_{\eps}\nrm*{\sum_{t=1}^{n}\eps_tX_t}_{\injt_4}\leq{}O\prn*{
\sqrt{\frac{n\log^{4}d}{d^{3/2}}} + \sqrt{\log{}d}
}.
\end{equation}
Additionally, if $n\leq{}d^{3}$, then the Rademacher complexity under
with-replacement sampling is bounded as
\begin{equation}
  \label{eq:tensor_completion_rademacher_wr}
  \En_{S\sim\dwr}\En_{\eps}\nrm*{\sum_{t=1}^{n}\eps_tX_t}_{\injt_4}\leq{}O\prn*{
\sqrt{\frac{n\log^{6}d}{d^{3/2}}} + \sqrt{\log{}^{3}d}
}. 
\end{equation}
\end{lemma}
\begin{proof}[\pfref{lem:tensor_completion_rademacher}]
We first bound the Rademacher complexity in the without-replacement sampling case, then handle the
with-replacement case by reduction. This analysis follows
\cite{barak2016noisy}, except that we handle certain
``diagonal'' terms that arise in the analysis slightly more carefully so as to get the
right scaling for our setup, which differs from theirs in that it does not assume incoherence.

Consider a fixed draw of $\eps_{1:n}$ and $X_{1:n}$, and let
$Z=\sum_{t=1}^{n}\eps_tX_t$. Leting $x,y,z$ be indeterminates, we will
give a degree-four \sosabbrev upper bound on the polynomial $\tri*{Z,x\tens{}y\tens{}z}$. This will imply that the
pseudoexpectation of $\tri*{Z,x\tens{}y\tens{}z}$ is bounded for any feasible pseudodistribution in the maximization
problem defining $\nrm*{\cdot}_{\injt_4}$.

To begin, for any fixed constant $\eta>0$, \pref{lem:pseudo_cs}
implies that 
\begin{align*}
\provable_{4}
  \tri*{Z,x\tens{}y\tens{}z}=\sum_{i,j,k}Z_{i,j,k}x_iy_jz_k
\leq{}\frac{1}{2\eta}\sum_{i=1}^{d}x_{i}^{2} + \frac{\eta}{2}\sum_{i}\prn*{\sum_{j,k}Z_{i,j,k}y_jz_k}^{2},
\end{align*}
and so 
\begin{align*}
\crl*{\nrm*{x}_{2}^{2}=1}\provable_{4}
  \tri*{Z,x\tens{}y\tens{}z}
&\leq{}\frac{1}{2\eta} + \frac{\eta}{2}\sum_{i}\prn*{\sum_{j,k}Z_{i,j,k}y_jz_k}^{2}.
\end{align*}
Define a matrix $A\in\bbR^{d^{2}\times{}d^{2}}$ via
  $A_{j,k',j',k}=\sum_{i=1}^{d}Z_{i,j,k}Z_{i,j',k'}$. Define
  additional matrices $D,B\in\bbR^{d^2\times{}d^2}$ via
\[
D_{j,k,j',k'} = 
\begin{cases} 
A_{j,k,j,k}, &\textrm{ if } (j,k) = (j',k'), \\ 
0, &\textrm{otherwise},
\end{cases}
\]
and $B=A-D$. Then we have
\begin{align*}
  \sum_{i}\prn*{\sum_{j,k}Z_{i,j,k}y_jz_k}^{2} =
  \tri*{A,(y\tens{}z)(y\tens{}z)^{\trn}}
= \tri*{B,(y\tens{}z)(y\tens{}z)^{\trn}} + \tri*{D,(y\tens{}z)(y\tens{}z)^{\trn}}.
\end{align*}
We bound the first term by the operator norm of $B$ via
\[
\crl*{\nrm*{y}_{2}^{2}=1,\nrm*{z}_{2}^{2}=1}\provable_4
\tri*{B,(y\tens{}z)(y\tens{}z)^{\trn}}
\leq{} \nrm*{B}_{\specnorm}\nrm*{y\tens{}z}_{\frobnorm}^{2}
= \nrm*{B}_{\specnorm}\nrm*{y}_{2}^{2}\nrm*{z}_{2}^{2}
\leq{} \nrm*{B}_{\specnorm}.
\]
For the second term, define another matrix $R\in\bbR^{d\times{}d}$ via
$R_{j,k}=\sum_{i=1}^{d}Z_{i,j,k}^{2}$. Then we can write
\[
\crl*{\nrm*{y}_{2}^{2}=1,\nrm*{z}_{2}^{2}=1}\provable_4\tri*{D,(y\tens{}z)(y\tens{}z)^{\trn}}=\sum_{j,k}R_{j,k}y_{j}^{2}z_{k}^{2}\leq{}\nrm*{R}_{\infty}\sum_{j,k}y_{j}^{2}z_{k}^{2}
=\nrm*{R}_{\infty}\nrm*{y}_{2}^{2}\nrm*{z}_{2}^{2}\leq{}\nrm*{R}_{\infty}.
\]
By the duality of sum-of-squares proofs and pseudodistributions, this implies
that any degree-four pseudodistribution $\mu$ with
$\mu\entails\crl*{\nrm*{x}_{2}^{2}=1,\nrm*{y}_{2}^{2}=1,\nrm*{z}_{2}^{2}=1}$ has
$\Ep\tri*{Z,x\tens{}y\tens{}z}\leq{}\frac{\eta}{2}(\nrm*{B}_{\specnorm}
+ \nrm*{R}_{\infty}) + \frac{1}{2\eta}$. Optimizing over $\eta$, we conclude that
\[
\nrm*{Z}_{\injt_4}\leq{}\sqrt{\nrm*{B}_{\specnorm} + \nrm*{R}_{\infty}}.
\]
The bound for the matrix $B$ is taken care of Theorem 4.4 of
\cite{barak2016noisy}, which implies that
$\sqrt{\En_{S\sim\dwo}\En_{\eps}\nrm*{B}_{\specnorm}^{2}}\leq{}O\prn*{\frac{n\log^{4}d}{d^{3/2}}}$. For
the matrix $R$, observe that under the without-replacement sampling model, for any $j,k$ we have
\[
R_{j,k}=\sum_{i=1}^{d}Z_{i,j,k}^{2}=\sum_{i=1}^{d}\ind\crl*{(i,j,k)\in\Omega}=\abs*{\crl*{t\in\brk*{n}:j_t=j,k_t=k}}.
\]
\pref{prop:bernstein_wo} thus implies that for any fixed $j,k$, with
probability at least $1-\delta$,
\[
R_{j,k}\leq{}O\prn*{n/d^{2} + \log(1/\delta)},
\]
and so, by taking a union bound and integrating out the tail, we have
\[
\sqrt{\En_{S\sim\dwo}\nrm*{R}_{\infty}^{2}}\leq{}O\prn*{n/d^{2} + \log{}d}.
\]
Combining the bounds on $B$ and $R$ and using Jensen's inequality yields
\begin{equation}
\label{eq:z_bound}
\En_{S\sim\dwo}\En_{\eps}\nrm*{Z}_{\injt_4}
\leq{}\sqrt{\En_{S\sim\dwo}\En_{\eps}\nrm*{Z}_{\injt_4}^{2}}\leq{} O\prn*{
\sqrt{\frac{n\log^{4}d}{d^{3/2}}} + \sqrt{\log{}d}
}.
\end{equation}
This completes the bound for the without-replacement case. For the
with-replacement case we reduce to the bound on the second moment above. Condition on the draw of $S$, and let
$T_{i,j,k}=\crl*{t\in\brk*{n}:i_t=i,j_t=j,k_t=k}$. Then we have
\[
\En_{\eps}\nrm*{\sum_{t=1}^{n}\eps_tX_t}_{\injt_4}
=
\En_{\eps}\nrm*{\sum_{(i,j,k)\in\Omega}\prn*{\sum_{t\in{}T_{i,j,k}}\eps_t}e_{i}\tens
  e_{j}\tens e_{k}}_{\injt_4}.
\]
Introduce a new sequence of Rademacher random variables
$\sigma\in\pmo^{d^{3}}$. Then the right-hand side above is equal to
\begin{align*}
\En_{\eps}\En_{\sigma}\nrm*{\sum_{(i,j,k)\in\Omega}\prn*{\sum_{t\in{}T_{i,j,k}}\eps_t}\sigma_{i,j,k}\cdot{}e_{i}\tens
  e_{j}\tens e_{k}}_{\injt_4}
&\leq   \En_{\eps}\max_{i,j,k}\abs*{\sum_{t\in{}T_{i,j,k}}\eps_t}\cdot\En_{\sigma}\nrm*{\sum_{(i,j,k)\in\Omega}\sigma_{i,j,k}\cdot{}e_{i}\tens
  e_{j}\tens e_{k}}_{\injt_4},
\end{align*}
where the inequality follows from the standard Lipschitz contraction lemma for
Rademacher complexity \citep{ledoux1991probability}. By the Hoeffding
bound, we have that for any fixed index $(i,j,k)$ with probability at least $1-\delta$,
\[
\abs*{\sum_{t\in{}T_{i,j,k}}\eps_t}\leq{}O\prn*{\max_{i,j,k}\sqrt{\abs*{T_{i,j,k}}\log{}(1/\delta)}}.
\]
Taking a union bound and integrating out the tail, we have
\[
\En_{\eps}\max_{i,j,k}\abs*{\sum_{t\in{}T_{i,j,k}}\eps_t}\leq{}O\prn*{\max_{i,j,k}\sqrt{\abs*{T_{i,j,k}}\log{}d}}.
\]
We now move to the final bound by taking the expectation over
$S$. Using Cauchy-Schwarz, the development above implies
\[
\En_{S\sim\dwr}\En_{\eps}\nrm*{\sum_{t=1}^{n}\eps_tX_t}_{\injt_4}
\leq\sqrt{\En_{S\sim\dwr}\abs*{T_{i,j,k}}\log{}d}\cdot{}\sqrt{\En_{S\sim\dwr}\En_{\sigma}\nrm*{\sum_{(i,j,k)\in\Omega}\sigma_{i,j,k}\cdot{}e_{i}\tens
  e_{j}\tens e_{k}}_{\injt_4}^{2}}.
\]
Bernstein's inequality and the union bound imply that
\[
\sqrt{\En_{S\sim\dwr}\abs*{T_{i,j,k}}\log{}d} \leq{} O\prn*{\sqrt{\log{}d(n/d^{3}+\log{}d)}}\leq{}O\prn*{\log{}d},
\]
where the second inequality uses that $n\leq{}d^{3}$. For the second
term, we have
\begin{align*}
&\En_{S\sim\dwr}\En_{\sigma}\nrm*{\sum_{(i,j,k)\in\Omega}\sigma_{i,j,k}\cdot{}e_{i}\tens
  e_{j}\tens e_{k}}_{\injt_4}^{2}\\
&=\sum_{m=1}^{n}\Pr_{S\sim{}\dwr}(\abs*{\Omega}=m) \cdot\En_{S\sim\dwr}\brk*{\En_{\sigma}\nrm*{\sum_{(i,j,k)\in\Omega}\sigma_{i,j,k}\cdot{}e_{i}\tens
  e_{j}\tens e_{k}}_{\injt_4}^{2}\biggr|\,\abs*{\Omega}=m}\\
&=\sum_{m=1}^{n}\Pr_{S\sim{}\dwr}(\abs*{\Omega}=m) \cdot\En_{S\sim\cD_{\wo}^{m}}\En_{\sigma}\nrm*{\sum_{(i,j,k)\in\Omega}\sigma_{i,j,k}\cdot{}e_{i}\tens
  e_{j}\tens e_{k}}_{\injt_4}^{2} \\
&\leq \sum_{m=1}^{n}\Pr_{S\sim{}\dwr}(\abs*{\Omega}=m) \cdot O\prn*{
\frac{m\log^{4}d}{d^{3/2}} + \log{}d
}.\\
&\leq O\prn*{
\frac{n\log^{4}d}{d^{3/2}} + \log{}d
},
\end{align*}
where the second inequality uses the bound we proved for the
without-replacement setting and the second inequality uses 
$m\leq{}n$ and that the probabilities sum to one. Combining these two
bounds completes the proof.

\end{proof}

\begin{lemma}
\label{lem:tensor_re_single_scale}
Define $\cT=\crl*{
T\in\RTthree\mid{}\nrm*{T}_{\nuct_4}\leq{}\tau,
\nrm*{\Sigma^{1/2}T}_{\frobnorm}\leq{}\beta, \nrm*{T}_{\infty}\leq{}R
}$. Under the with-replacement sampling model, when
$d^{3/2}\leq{}n\leq{}d^{3}$, we have that for any $\delta>0$, with
probability at least $1-\delta$,
\[
\abs*{\frac{1}{n}\nrm*{\dataop(T)}_{\frobnorm}^{2}-\nrm*{\Sigma^{1/2}T}_{\frobnorm}^{2}}\leq 
O\prn*{
R\tau\sqrt{\frac{\log^{6}d}{nd^{3/2}}} 
+
\sqrt{
\frac{2R^{2}\beta^{2}\log(1/\delta)}{n}
}
+ \frac{2R^{2}\log(1/\delta)}{n}
}\quad\forall{}T\in\cT.
\]
\end{lemma}
\begin{proof}[\pfref{lem:tensor_re_single_scale}]
To begin, we write
\[
\abs*{\frac{1}{n}\nrm*{\dataop{}(T)}_{2}^{2}-\nrm*{\Sigma^{1/2}T}_{\frobnorm}^{2}}
= \abs*{\Eh\tri*{T,X}^{2} - \En\tri*{T,X}^{2}}.
\]
Using this representation, since entries are drawn i.i.d. with
replacement, we can apply \pref{lem:talagrand} with the function class
$\cF=\crl*{X\mapsto{}\tri*{T,X}^{2}-\En\tri*{T,X}^{2}\mid{}T\in\cT}$. In
particular, note that for all $T\in\cT$ we have
$\abs*{\tri{T,X_t}}=\abs*{T_{i_t,j_t,k_t}}\leq{}R$, and furthermore 
\[
\sup_{T\in\cT}\En\prn*{\tri*{T,X}^{2}-\En\tri*{T,X}^{2}}^{2}\leq{} \sup_{T\in\cT}\En\tri*{T,X}^{4}
\leq{} R^{2}\sup_{T\in\cT}\En\tri*{T,X}^{2}= R^{2}\sup_{T\in\cT}\nrm*{\Sigma^{1/2}T}_{\frobnorm}^{2}\leq{}R^{2}\beta^{2}.
\]

Consequently, \pref{lem:talagrand} implies that with probability at least $1-\delta$, for all $T\in\cT$,
\[
\abs*{\frac{1}{n}\nrm*{\dataop{}(T)}_{2}^{2}-\nrm*{\Sigma^{1/2}T}_{\frobnorm}^{2}}\leq  4\En_{S}\En_{\eps}\sup_{T\in\cT}\frac{1}{n}\sum_{t=1}^{n}\eps_{t}(\tri*{T,X_t}^{2}-\En\tri*{T,X}^{2})
+ \sqrt{
\frac{2R^{2}\beta^{2}\log(1/\delta)}{n}
}
+ \frac{2R^{2}\log(1/\delta)}{n}.
\]
Using Jensen's inequality and splitting the supremum, we have
\[
\En_{S}\En_{\eps}\sup_{T\in\cT}\frac{1}{n}\sum_{t=1}^{n}\eps_{t}(\tri*{T,X_t}^{2}-\En\tri*{T,X}^{2})
\leq {} 2\En_{S}\En_{\eps}\sup_{T\in\cT}\frac{1}{n}\sum_{t=1}^{n}\eps_{t}\tri*{T,X_t}^{2}.
\]

Using the Lipschitz contraction lemma for Rademacher complexity, we
remove the square
\[
\En_{S}\En_{\eps}\sup_{T\in\cT}\brk*{\frac{1}{n}\sum_{t=1}^{n}\eps_{t}\tri*{T,X_t}^{2}}
\leq{} 2R\En_{S}\En_{\eps}\sup_{T\in\cT}\brk*{\frac{1}{n}\sum_{t=1}^{n}\eps_{t}\tri*{T,X_t}}
\leq{}\frac{2R\tau}{n}\En_{S}\En_{\eps}\nrm*{\frac{1}{n}\sum_{t=1}^{n}\eps_tX_t}_{\injt_4}.
\]
Finally, using \pref{lem:tensor_completion_rademacher}, we have
\[
\En_{S}\En_{\eps}\nrm*{\frac{1}{n}\sum_{t=1}^{n}\eps_tX_t}_{\injt_4} \leq{} O\prn*{
\sqrt{\frac{\log^{6}d}{nd^{3/2}}} + \frac{\log{}^{3/2}d}{n}
}
\leq{} O\prn*{
\sqrt{\frac{\log^{6}d}{nd^{3/2}}} + \frac{\log{}^{3/2}d}{n}
}.
\]
The final bound follows by using $n\geq{}d^{3/2}$ to simplify this
expression.

\end{proof}

\begin{proof}[\pfref{thm:tensor_completion_re}]
Let $\cT=\crl*{T\in\RTthree\mid{}\nrm*{T}_{\infty}\leq{}R}$. Recall that for tensor completion, the population correlation matrix $\Sigma$ under with-replacement sampling is equal to $\frac{1}{d^{3}}I$.

Let $\tau_{\max}=Rd^{3}$, $\beta_{\max}=Rd^{3/2}$, $\tau_{\min}=Rd^{3/2}/\sqrt{n}$, $\beta_{\min}=R/\sqrt{n}$, and let $N=\ceil{\log(\tau_{\max}/\tau_{\min})}+1$ and $M=\ceil{\log(\beta_{\max}/\beta_{\min})}+1$. For each $i\in\brk*{N}$ and $j\in\brk*{M}$ define $\tau_i=\tau_{\max}e^{1-i}$ and $\beta_j=\beta_{\max}e^{1-j}$. Define
\[
\cT_{i,j} = \crl*{
T\in\cT\mid{} \tau_{i+1}\leq{}\nrm*{T}_{\nuct_4}\leq{}\tau_{i},\;\;\beta_{j+1}\leq{}\nrm*{\Sigma^{1/2}T}_{\frobnorm}\leq{}\beta_{j}
}.
\]
Using \pref{lem:tensor_re_single_scale} and a union bound, we get that with probability at least $1-\delta$, for all $i,j$ simultaneously, for all $T\in\cT_{i,j}$,
\[
\nrm*{\Sigma^{1/2}T}_{\frobnorm}^{2}\leq \frac{1}{n}\nrm*{\dataop{}(T)}_{2}^{2} + 
O\prn*{
R\tau_i\sqrt{\frac{\log^{6}d}{nd^{3/2}}} 
+
\sqrt{
\frac{R^{2}\beta^{2}_j\log(MN/\delta)}{n}
}
+ \frac{R^{2}\log(MN/\delta)}{n}
}.
\]
Now consider a fixed tensor $T\in\cT$. There are two cases. First, if $\nrm*{T}_{\nuct_4}\geq{}\tau_{\min}$ and $\nrm*{\Sigma^{1/2}T}_{\frobnorm}\geq\beta_{\min}$, then there must be indices $i$ and $j$ for which $\tau_{i+1}\leq{}\nrm*{T}_{\nuc}\leq{}\tau_{i},\;\;\beta_{j+1}\leq{}\nrm*{\Sigma^{1/2}T}_{\frobnorm}\leq{}\beta_{j}$. Consequently, the uniform bound above implies
\[
\nrm*{\Sigma^{1/2}T}_{\frobnorm}^{2}\leq \frac{1}{n}\nrm*{\dataop{}(T)}_{2}^{2} + 
O\prn*{
R\nrm*{T}_{\nuct_4}\sqrt{\frac{\log^{6}d}{nd^{3/2}}} 
+
\nrm*{\Sigma^{1/2}T}_{\frobnorm}\sqrt{
\frac{R^{2}\log(MN/\delta)}{n}
}
+ \frac{R^{2}\log(MN/\delta)}{n}
}.
\]
On the other hand, if either $\nrm*{T}_{\nuc}\leq\tau_{\min}$ or $\nrm*{\Sigma^{1/2}T}_{\frobnorm}\leq\beta_{\min}$, we trivially have
\[
\nrm*{\Sigma^{1/2}T}_{\frobnorm}^{2}\leq \frac{R^{2}}{n} \leq{} \frac{1}{n}\nrm*{\dataop(T)}_{2}^{2} + \frac{R^{2}}{n}.
\]
Combining these cases, and using the values for $N$ and $M$, we get that with probability at least $1-\delta$, for all $T\in\cT$,
\begin{align*}
\nrm*{\Sigma^{1/2}T}_{\frobnorm}^{2} &\leq
  \frac{1}{n}\nrm*{\dataop{}(T)}_{2}^{2} + 
O\prn*{
R\nrm*{T}_{\nuct_4}\sqrt{\frac{\log^{6}d}{nd^{3/2}}} 
+
\nrm*{\Sigma^{1/2}T}_{\frobnorm}\sqrt{
\frac{R^{2}\log(\log^{2}(d^{3/2}\sqrt{n})/\delta)}{n}
}}\\
&~~~~+ O\prn*{\frac{R^{2}\log(\log^{2}(d^{3/2}\sqrt{n})/\delta)}{n}
}.
\end{align*}
Using the AM-GM inequality on the second-to-last term and rearranging, we have that for any $\veps>0$,
\[
(1-\veps)\nrm*{\Sigma^{1/2}T}_{\frobnorm}^{2}\leq \frac{1}{n}\nrm*{\dataop{}(T)}_{2}^{2} + 
O\prn*{
R\nrm*{T}_{\nuct_4}\sqrt{\frac{\log^{6}d}{nd^{3/2}}} 
+ \frac{R^{2}\log(\log^{2}(d^{3/2}\sqrt{n})/\delta)}{\veps{}n}
}.
\]
When $\veps\in(0,1/2)$, this implies
\[
\nrm*{\Sigma^{1/2}T}_{\frobnorm}^{2}\leq \frac{(1+2\veps)}{n}\nrm*{\dataop{}(T)}_{2}^{2} + 
O\prn*{
R\nrm*{T}_{\nuct_4}\sqrt{\frac{\log^{6}d}{nd^{3/2}}} 
+ \frac{R^{2}\log(\log^{2}(d^{3/2}\sqrt{n})/\delta)}{\veps{}n}
}.
\]
\end{proof}

\subsection{Restricted eigenvalue for tensor sensing}
\label{app:re_tensor_sensing}

In this section we prove the main technical result used to establish
restricted eigenvalue guarantees for tensor sensing, which is as follows.
\begin{theorem}
\label{thm:tensor_sensing_re}
Suppose $X\sim{}\cN(0\midsem{}I)$. There is some universal constant
$C>0$ such that for any $\veps<\sqrt{2/\pi}$, with probability at least $1-2e^{-\frac{n}{32}}/\prn{1-e^{-\frac{n\veps^{2}}{8}}}$,
\begin{equation}
\label{eq:tensor_sensing_re}
\frac{1}{\sqrt{n}}\nrm*{\dataop(\Delta)}_{2} \geq{}  (\sqrt{2/\pi}-\veps)\nrm*{\Delta}_{\frobnorm}-\frac{Cd^{3/4}\log^{1/4}d}{\sqrt{n}}\cdot\nrm*{\Delta}_{\nuct_{4}}\quad\quad\forall{}\Delta\in\RTthree.
\end{equation}
\end{theorem}
To prove \pref{thm:tensor_sensing_re} we require two key technical lemmas.
\begin{lemma}
\label{lem:gaussian_distance_conc}
Let $\nrm*{\cdot}$ be any norm, let $\nrm*{\cdot}_{\star}$ be the
dual. Suppose that the rows of $\dataop$ are formed by drawing
$X_1,\ldots,X_{n}$ i.i.d. from $\cN(0\midsem{}I_{p\times{}p})$, and
let
$\En_{X\sim{}\cN(0\midsem{}I_{p\times{}p})}\nrm*{X}\leq{}\psi$. Then
for any $\veps<\sqrt{2/\pi}$, with probability at least $1-2e^{-\frac{n}{32}}/\prn{1-e^{-\frac{n\veps^{2}}{8}}}$,
\begin{equation}
\label{eq:gaussian_distance_conc}
\frac{1}{\sqrt{n}}\nrm*{\dataop(\Delta)}_{2} \geq{} (\sqrt{2/\pi}-\veps)\nrm*{\Delta}_{\frobnorm}-\frac{4\psi}{\sqrt{n}}\nrm*{\Delta}_{\star}\quad\quad\forall{}\Delta\in\bbR^{p}.
\end{equation}
\end{lemma}

\begin{lemma}[Corollary of \cite{hopkins2015tensor}, Theorem 3.3]
\label{lem:sos_gaussian}
Let $X\in\RTthree$ have entries drawn i.i.d. from $\cN(0\midsem{}1)$. Then
\begin{equation}
\En\nrm*{X}_{\injt_4}\leq{}O\prn*{d^{3/4}\log^{1/4}d}.
\end{equation}
\end{lemma}

\begin{proof}[\pfref{thm:tensor_sensing_re}]
This is an immediate consequence of \pref{lem:gaussian_distance_conc}
and \pref{lem:sos_gaussian}, along with the duality of
$\nrm*{\cdot}_{\injt_4}$ and $\nrm*{\cdot}_{\nuct_4}$.
\end{proof}

In the remainder of the section we prove
\pref{lem:gaussian_distance_conc}. The result follows from fairly
standard techniques (e.g. \cite{wainwright2019high}), but we include
the proof for completeness. We first restate some basic results on
gaussian concentration.

\begin{lemma}[Concentration for Lipschitz functions \citep{milman1986asymptotic}]
\label{lem:gaussian_lipschitz}
Let $Z\in\bbR^{n}$ have entries drawn i.i.d. from $\cN(0\midsem{}1)$,
and let $f:\bbR^{n}\to\bbR$ be $L$-Lipschitz with respect to the
$\ls_{2}$ norm. Then for all $t\geq{}0$,
\begin{equation*}
\Pr\prn*{\abs*{f(Z)-\En{}f(Z)} \geq{} t} \leq{} 2e^{-\frac{t^{2}}{2L^{2}}}.
\end{equation*}
\end{lemma}

\begin{lemma}[Gordon's Inequality \citep{davidson2001local}]
\label{lem:gordon}
Let $\crl*{Z_{a,b}}$ and $\crl*{Y_{a,b}}$ be zero-mean gaussian processes indexed by $A\times{}B$. Suppose that
\[
\En\prn*{Z_{a,b}-Z_{a',b'}}^{2} \leq{} \En\prn*{Y_{a,b}-Y_{a',b'}}^{2}\quad\quad\text{for all $(a,b),(a',b')\in{}A\times{}B$}
\]
and 
\[
\En\prn*{Z_{a,b}-Z_{a,b'}}^{2} = \En\prn*{Y_{a,b}-Y_{a,b'}}^{2}\quad\quad\text{for all $a\in{}A$, $b,b'\in{}B$}.
\]
Then
\[
\En\sup_{a\in{}A}\inf_{b\in{}B}Z_{a,b} \leq{} \En\sup_{a\in{}A}\inf_{b\in{}B}Y_{a,b}.
\]

\end{lemma}

\begin{proof}[\pfref{lem:gaussian_distance_conc}]\emph{Part 1: Bound at a single scale.}

Define $\cB(\tau) = \crl*{\Delta\in\bbR^{p}\mid{}\nrm*{\Delta}_{\frobnorm}=1,\nrm*{\Delta}_{\star}\leq{}\tau}$. We will prove that with probability at least $1-2e^{-\frac{nt^{2}}{2}}$,
\begin{equation}
\label{eq:sensing_conc_single_scale}
\frac{1}{\sqrt{n}}\nrm*{\dataop(\Delta)}_{2} \geq{} \sqrt{\frac{2}{\pi}} - \frac{\psi\tau}{\sqrt{n}}-t,\quad\quad\forall{}\Delta\in\cB(\tau).
\end{equation}
We will prove a lower bound on the random variable
$\min_{\Delta\in\cB(\tau)}\frac{1}{\sqrt{n}}\nrm*{\dataop(\Delta)}_{2}$,
which is equivalent to providing an upper bound on the random variable
\[
S_{n}(\tau) = -\inf_{\Delta\in\cB(\tau)}\frac{1}{\sqrt{n}}\nrm*{\dataop(\Delta)}_{2}
= -\inf_{\Delta\in\cB(\tau)}\sup_{\nrm*{u}_{2}=1}\frac{1}{\sqrt{n}}\tri*{u,\dataop(\Delta)}
= \sup_{\Delta\in\cB(\tau)}\inf_{\nrm*{u}_{2}=1}\frac{1}{\sqrt{n}}\tri*{u,\dataop(\Delta)}.
\]
As a starting point, we show how to upper bound the expectation
$\En\brk*{S_n(\tau)}$. Define
$Z_{\Delta,u}=\frac{1}{\sqrt{n}}\tri*{u,\dataop(\Delta)}$, so that
$S_{n}(\tau)=\sup_{\Delta\in\cB(\tau)}\inf_{\nrm*{u}_{2}=1}Z_{\Delta,u}$. Note
that $Z_{\Delta,u}$ is a gaussian process with variance $n^{-1}$,
since  $\nrm*{\Delta}_{\frobnorm}=1$. We now define a new gaussian process that will serve as an upper bound through Gordon's inequality. Let $g\in\bbR^{p}$ and $h\in\bbR^{n}$ be standard gaussian random variables, and define
\[
Y_{\Delta,u} = \frac{1}{\sqrt{n}}\tri*{g,\Delta} + \frac{1}{\sqrt{n}}\tri*{h,u}.
\]
Note that for any $(\Delta,u)$ and $(\Delta',u')$ we have 
\[
\En(Y_{\Delta,u}-Y_{\Delta',u'})^{2} = \frac{1}{n}\nrm*{\Delta-\Delta'}_{\frobnorm}^{2} + \frac{1}{n}\nrm*{u-u'}_{2}^{2}.
\]
Interpreting $\Delta$ and $\Delta'$ as vectors in $\bbR^{d^{3}}$, we also have
\begin{align*}
\En(Z_{\Delta,u}-Z_{\Delta',u'})^{2} &= \frac{1}{n}\nrm*{u\Delta^{\trn}-u'\Delta'^{\trn}}_{\frobnorm}^{2} \\
&= \frac{1}{n}\nrm*{\Delta-\Delta'}_{\frobnorm}^{2} + \frac{1}{n}\nrm*{u-u'}_{2}^{2}
+ \frac{1}{n}\tri*{(u-u')\Delta^{\trn},u'(\Delta-\Delta')^{\trn}} \\
&= \frac{1}{n}\nrm*{\Delta-\Delta'}_{\frobnorm}^{2} + \frac{1}{n}\nrm*{u-u'}_{2}^{2}
+ \frac{1}{n}(\tri{u,u'}-1)(1-\tri{\Delta,\Delta'}) \\
&\leq{} \frac{1}{n}\nrm*{\Delta-\Delta'}_{\frobnorm}^{2} + \frac{1}{n}\nrm*{u-u'}_{2}^{2},
\end{align*}
where we have used that
$\nrm*{u}_{2}=\nrm*{u'}_{2}=\nrm*{\Delta}_{\frobnorm}=\nrm*{\Delta'}_{\frobnorm}=1$. It
is also easily seen from the representation above that for any triple
$(\Delta,u,u')$ we have equality: $\En(Z_{\Delta,u}-Z_{\Delta,u'})^{2}=\En(Y_{\Delta,u}-Y_{\Delta,u'})^{2}$.
This means that the preconditions of \pref{lem:gordon} are satisfied, and so
\begin{align*}
\En\brk*{S_n(\tau)} &\leq{} \En\sup_{\Delta\in\cB(\tau)}\inf_{\nrm*{u}_{2}=1}Y_{\Delta,u} \\
&= \frac{1}{\sqrt{n}}\En_{g}\sup_{\Delta\in\cB(\tau)}\tri*{g,\Delta} + \En_{h}\frac{1}{\sqrt{n}}\min_{\nrm*{u}_{2}=1}\tri*{h,u} \\
&= \frac{\tau}{\sqrt{n}}\cdot{}\En_{g}\nrm*{g} - \frac{1}{\sqrt{n}}\En_{h}\nrm*{h}_{2}\\
&\leq{} \frac{\psi\tau}{\sqrt{n}} - \sqrt{\frac{2}{\pi}}.
\end{align*}
The provides the desired upper bound in expectation. To establish the
high probability result we appeal to gaussian concentration for
Lipschitz functions. Let $\bX\in\bbR^{n\times{}d^{3}}$ denote the sequence of
measurements $X_1,\ldots,X_n$, interpreted as a matrix with vectorized
measurements as rows. Define $f:\bbR^{n\times{}d^{3}}\to\bbR$ via $f(\bX) =
\min_{\Delta\in\cB(\tau)}\frac{1}{\sqrt{n}}\nrm*{\bX\Delta}_{2}$. Observe
that we have
\[
\abs*{f(\bX) - f(\bX')} \leq{} \sup_{\Delta\in\cB(\tau)}\frac{1}{\sqrt{n}}\nrm*{(\bX-\bX')\Delta}_{2}
\leq{} \frac{1}{\sqrt{n}}\nrm*{\bX-\bX'}_{\op}
\leq{} \frac{1}{\sqrt{n}}\nrm*{\bX-\bX'}_{F},
\]
and so $f$ is $n^{-\frac{1}{2}}$-Lipschitz with respect to
$\ls_{2}$. \pref{lem:gaussian_lipschitz} therefore implies that 
\[
\Pr\prn*{\abs*{S_{n}(\tau)-\En{}S_n(\tau)} \geq{} t} \leq{} 2e^{-\frac{nt^{2}}{2}}.
\]
In other words, \pref{eq:sensing_conc_single_scale} holds.

\emph{Part 2: Bound at all scales.}
We will show that for any $\veps<\sqrt{2/\pi}$, with probability at least $1-2e^{-\frac{n}{32}}/\prn{1-e^{-\frac{n\veps^{2}}{8}}}$,
\begin{equation}
\label{eq:sensing_conc_multi_scale}
\frac{1}{\sqrt{n}}\nrm*{\dataop(\Delta)}_{2} \geq{} \sqrt{\frac{2}{\pi}}-\veps-\frac{4\psi}{\sqrt{n}}\nrm*{\Delta}_{\star},\quad\quad\forall{}\Delta:\nrm*{\Delta}_{\frobnorm}=1.
\end{equation}
Define $\cB(\tau_{\ls},\tau_{u})  =
\crl*{\Delta\in\bbR^{p}\mid{}\nrm*{\Delta}_{\frobnorm}=1,
  \tau_{\ls}\leq\frac{\psi}{\sqrt{n}}\nrm*{\Delta}\leq{}\tau_{u}}$.
Set $\mu=\veps/2$ and consider the classes $\cB(0,\mu)$ and
$\cB(2^{i-1}\mu,2^{i}\mu)$ for $i\in\bbN$. Note that if equation
\pref{eq:sensing_conc_multi_scale} fails to hold and $\Delta\in\cB(0,\mu)$ then
\[
\frac{1}{\sqrt{n}}\nrm*{\dataop(\Delta)}_{2} \leq \sqrt{\frac{2}{\pi}}-\veps-\frac{4\psi}{\sqrt{n}}\nrm*{\Delta}_{\star}\leq{}\sqrt{\frac{2}{\pi}}-\veps=\sqrt{\frac{2}{\pi}}-2\mu.
\]
Furthermore, if equation \pref{eq:sensing_conc_multi_scale} fails to hold and $\Delta\in\cB(2^{i-1}\mu,2^{i}\mu)$, then 
\[
\frac{1}{\sqrt{n}}\nrm*{\dataop(\Delta)}_{2} \leq \sqrt{\frac{2}{\pi}}-\veps-\frac{4\psi}{\sqrt{n}}\nrm*{\Delta}_{\star}
\leq \sqrt{\frac{2}{\pi}}-2\cdot2^{i}\mu.
\]
Our development in part 1 of the proof implies that for any fixed $\tau_{\ls}\leq{}\tau_{u}$, with probability at least $1-2e^{-\frac{n}{32}}e^{-\frac{\tau_{u}^{2}}{2}}$,
\begin{equation}
\label{eq:sensing_conc_multi_scale_i}
\frac{1}{\sqrt{n}}\nrm*{\dataop(\Delta)}_{2} \geq{} \sqrt{\frac{2}{\pi}}-2\tau_{u},
\quad\quad\forall{}\Delta\in\cB(\tau_{\ls},\tau_{u}).
\end{equation}
Thus, by a union bound, we get that with probability at least
$1-2e^{-\frac{n}{32}}\sum_{i=0}^{\infty}e^{-\frac{2^{2i}n\veps^{2}}{8}}$,
or conservatively at least $1-2e^{-\frac{n}{32}}/\prn{1-e^{-\frac{n\veps^{2}}{8}}}$,
\[
\frac{1}{\sqrt{n}}\nrm*{\dataop(\Delta)}_{2} \geq{} \sqrt{\frac{2}{\pi}}-2\cdot2^{i}\mu,
\quad\quad\forall{}\Delta\in\cB(2^{i-1}\mu,2^{i}\mu),\quad\forall{}i\in\bbN,
\]
and
\[
\frac{1}{\sqrt{n}}\nrm*{\dataop(\Delta)}_{2} \geq{} \sqrt{\frac{2}{\pi}}-2\mu,
\quad\quad\forall{}\Delta\in\cB(0,\mu),
\]
or in other words, equation \pref{eq:sensing_conc_multi_scale} holds.

We extend the guarantee in equation \pref{eq:sensing_conc_multi_scale} to arbitrary $\Delta\in\bbR^{p}$ by rescaling so that $\nrm*{\Delta}_{\frobnorm}=1$ and then exploiting homogeneity to get
\[
\frac{1}{\sqrt{n}}\nrm*{\dataop(\Delta)}_{2} \geq{} \prn*{\sqrt{\frac{2}{\pi}}-\veps}\nrm*{\Delta}_{\frobnorm}-\frac{4\psi}{\sqrt{n}}\nrm*{\Delta}_{\star},\quad\forall\Delta\in\bbR^{p}.
\]
\end{proof}

\subsection{Proofs of main results}
\label{app:main_theorems}

\begin{proof}[\pfref{thm:tensor_completion}]
We prove the theorem by appealing to the generic result of \pref{thm:generic_offset}
using norms $\nrm*{\cdot}=\nrm*{\cdot}_{\injt_4}$ and
$\nrm*{\cdot}_{\star}=\nrm*{\cdot}_{\nuct_4}$. Let $T^{\star}$ be an
arbitrary rank-$r$ orthogonal tensor with $\nrm*{T^{\star}}_{\nuct_6}=\tau$
and $\nrm*{T^{\star}}_{\infty}\leq{}R$.

First, observe that all elements $T\in\cT$ have
$\nrm*{T}_{\nuct_6}\leq{}\nrm*{\ts}_{\nuct_6}$. Consequently,
\pref{thm:sos_norm_comparison} implies that all $\Delta\in\cT-\ts$
satisfy
\begin{equation}
\label{eq:tensor_comp_proof_comparison}
\nrm*{\Delta}_{\nuct_4}\leq{}68r\cdot\nrm*{\Delta}_{\frobnorm},
\end{equation}

which establishes \propthree with $\kappa=O(r\cdot{}d^{3/2})$.

To establish \proptwo we appeal to \pref{thm:tensor_completion_re},
which implies that for any $\veps<1/2$, with probability at least $1-\delta$,
all $\Delta\in\cT-\ts$ satisfy
\[
\nrm*{\Sigma^{1/2}\Delta}_{\frobnorm}^{2}\leq \frac{(1+2\veps)}{n}\nrm*{\dataop{}(\Delta)}_{2}^{2} + 
O\prn*{
R\nrm*{\Delta}_{\nuct_4}\sqrt{\frac{\log^{6}d}{nd^{3/2}}} 
+ \frac{R^{2}\log(\log^{2}(d^{3/2}\sqrt{n})/\delta)}{\veps{}n}
}.
\]
Using equation \pref{eq:tensor_comp_proof_comparison} this is upper bounded by 
\begin{align*}
\nrm*{\Sigma^{1/2}\Delta}_{\frobnorm}^{2}
&\leq \frac{(1+2\veps)}{n}\nrm*{\dataop{}(\Delta)}_{2}^{2} + O\prn*{
R\nrm*{\Delta}_{\frobnorm}\sqrt{\frac{r^{2}\log^{6}d}{nd^{3/2}}} 
+ \frac{R^{2}\log(\log^{2}(d^{3/2}\sqrt{n})/\delta)}{\veps{}n}
} \\
&= \frac{(1+2\veps)}{n}\nrm*{\dataop{}(\Delta)}_{2}^{2} + O\prn*{
R\nrm*{\Sigma^{1/2}\Delta}_{\frobnorm}\sqrt{\frac{r^{2}d^{3/2}\log^{6}d}{n}} 
+ \frac{R^{2}\log(\log^{2}(d^{3/2}\sqrt{n})/\delta)}{\veps{}n}
}.
\end{align*}
Using the AM-GM inequality, this is further upper bounded by
\[
\nrm*{\Sigma^{1/2}\Delta}_{\frobnorm}^{2} \leq 
\frac{(1+2\veps)}{n}\nrm*{\dataop{}(\Delta)}_{2}^{2} +
\veps\nrm*{\Sigma^{1/2}\Delta}_{\frobnorm}^{2} + O\prn*{
R^{2}\frac{r^{2} d^{3/2}\log^{6}d}{n} 
+ \frac{R^{2}\log(\log^{2}(d^{3/2}\sqrt{n})/\delta)}{\veps{}n}
}.
\]
Rearranging, this is equivalent to
\[
(1-\veps)\nrm*{\Sigma^{1/2}\Delta}_{\frobnorm}^{2} \leq 
\frac{(1+2\veps)}{n}\nrm*{\dataop{}(\Delta)}_{2}^{2} + O\prn*{
R^{2}\frac{r^{2} d^{3/2}\log^{6}d}{n} 
+ \frac{R^{2}\log(\log^{2}(d^{3/2}\sqrt{n})/\delta)}{\veps{}n}
},
\]
and since $\veps<1/2$ this implies
\[
\nrm*{\Sigma^{1/2}\Delta}_{\frobnorm}^{2} \leq 
\frac{(1+2\veps)^{2}}{n}\nrm*{\dataop{}(\Delta)}_{2}^{2} +
O\prn*{
R^{2}\frac{r^{2} d^{3/2}\log^{6}d}{n} 
+ \frac{R^{2}\log(\log^{2}(d^{3/2}\sqrt{n})/\delta)}{\veps{}n}
}.
\]
Making the somewhat arbitrary choice of $\veps=1/100$,  and
simplifying the right-hand-side, we get
\[
\nrm*{\Sigma^{1/2}\Delta}_{\frobnorm}^{2} \leq 
\frac{1.1}{n}\nrm*{\dataop{}(\Delta)}_{2}^{2} +
O\prn*{
R^{2}\frac{r^{2} d^{3/2}\log^{6}d+\log(1/\delta)}{n} 
}.
\]
So we can take $c=1.1$ and $\gamma_n = O\prn*{
R^{2}\frac{r^{2} d^{3/2}\log^{6}d+\log(1/\delta)}{n} 
}$.

Finally, we establish \propone. \pref{lem:talagrand} establishes that
with probability at least $1-\delta$,
\[
\nrm*{\sum_{t=1}^{n}\xi_tX_t-\En\brk*{\xi{}X}}_{\injt_4}
\leq{} 4\En_{S}\En_{\eps}\nrm*{\sum_{t=1}^{n}\eps_t(\tri*{\ts,X_t}-Y_t)X_t}_{\injt_4}
+ \sqrt{
  2\sigma^{2}n\log(1/\delta)}
+ 4R\log(1/\delta),
\]
where
$\sigma^{2}=\sup_{\nrm*{T}_{\nuct_4}\leq{}1}\En\brk*{\tri*{T,X}^{2}(\tri*{\ts,X_t}-Y_t)^{2}}\leq{}\sup_{\nrm*{T}_{\nuct_4}\leq{}1}\frac{4R^{2}}{d^{3}}\nrm*{T}_{\frobnorm}^{2}\leq{}\frac{4R^{2}}{d^{3}}$. Using
\pref{lem:tensor_completion_rademacher} with the 
standard in-expectation Lipschitz contraction lemma (e.g.,
\cite{ledoux1991probability}), we have
\begin{align*}
4\En_{S}\En_{\eps}\nrm*{\sum_{t=1}^{n}\eps_t(\tri*{\ts,X_t}-Y_t)X_t}_{\injt_4}
\leq{}8R \En_{S}\En_{\eps}\nrm*{\sum_{t=1}^{n}\eps_tX_t}_{\injt_4}
\leq{}O\prn*{
R\sqrt{\frac{n\log^{6}d}{d^{3/2}}} + R\sqrt{\log{}^{3}d}
}. 
\end{align*}
Since $n=\Omega(d^{3/2})$,  these bounds together imply that we can take
$M=O\prn*{R\sqrt{\frac{\log^{6}d+\log(1/\delta)}{d^{3/2}}}}$.

\pref{thm:generic_offset} now implies the claimed result.

\end{proof}
%\subsection{Proof of \pref{thm:tensor_sensing}}

\begin{proof}[\pfref{thm:tensor_sensing}]
As in the tensor completion case, we prove the theorem by appealing to \pref{thm:generic_offset}
using norms $\nrm*{\cdot}=\nrm*{\cdot}_{\injt_4}$ and $\nrm*{\cdot}_{\star}=\nrm*{\cdot}_{\nuct_4}$.

Let $T^{\star}$ be an arbitrary orthogonal tensor with
$\nrm*{T^{\star}}_{\nuct_6}=\tau$, $R(\ts)=R$, and $r(\ts)=r$. \pref{thm:sos_norm_comparison}
implies that all $\Delta\in\cT-\ts$ satisfy
\begin{equation}
\nrm*{\Delta}_{\nuct_4}\leq{}68r\cdot\nrm*{\Delta}_{\frobnorm}\label{eq:comparison_tensor_sensing}
\end{equation}

 and thus, since $\Sigma=I$, we may take $\kappa=68r$ to establish
\propthree.

To establish \proptwo we appeal to \pref{thm:tensor_sensing_re}. This
implies that for any $\veps<\sqrt{2/\pi}$, with probability at least
$1-2e^{-\frac{n}{32}}/\prn{1-e^{-\frac{n\veps^{2}}{8}}}$,
\[
\frac{1}{\sqrt{n}}\nrm*{\dataop(\Delta)}_{2} \geq{}  (\sqrt{2/\pi}-\veps)\nrm*{\Delta}_{\frobnorm}-\frac{Cd^{3/4}\log^{1/4}d}{\sqrt{n}}\cdot\nrm*{\Delta}_{\nuct_{k}},
\]
where $C>0$ is some absolute constant. Using equation \pref{eq:comparison_tensor_sensing}, we have that
for all $\Delta\in\cT-\ts$, conditioned on the event above,
\[
\frac{1}{\sqrt{n}}\nrm*{\dataop(\Delta)}_{2} \geq{}  (\sqrt{2/\pi}-\veps)\nrm*{\Delta}_{\frobnorm}-\frac{C'r d^{3/4}\log^{1/4}d}{\sqrt{n}}\cdot\nrm*{\Delta}_{\frobnorm}.
\]
This means that when
$n=\Omega(r^{2}(T^{\star})d^{3/2}\log^{1/2}d/\veps)$, we have
\[
\nrm*{\Delta}_{\frobnorm}^{2}\leq{}\frac{1}{ (\sqrt{2/\pi}-2\veps)^{2}}\frac{1}{n}\nrm*{\dataop(\Delta)}_{2}^{2}.
\]
It suffices to set $\veps=1/40$ to get
\[
\nrm*{\Delta}_{\frobnorm}^{2}\leq{}
  \frac{1.8}{n}\nrm*{\dataop(\Delta)}_{2}^{2}.
\]
So, simplifying,  \proptwo is satisfied with $c=1.8$ and $\gamma_n=0$ with
probability at least $1-e^{-\frac{n}{64}}$ when $n=\Omega(r^{2}(T^{\star})d^{3/2}\log^{1/2}d)$.

To establish \propone we use \pref{lem:sos_gaussian}, but some care needs to be taken to establish
that this applies with high probability. Pick a constant
$\tau\geq{}0$, and observe via \pref{lem:gaussian_lipschitz} that
\[
\Pr\prn*{\nrm*{X}_{F}\geq{}d^{3/2}+\tau} \leq{} 2e^{-\frac{\tau^{2}}{2}}.
\]
Define a truncated sequence
$X'_t=X_t\ind\crl*{\nrm*{X_t}_{\frobnorm}\leq{}d^{3/2}+\tau}$, and set
  $\delta_0=2ne^{-\frac{\tau^{2}}{2}}$. Observe that with probability at least
  $1-\delta_0$, $X_t=X'_t$ for all $t$, and so
\[
\nrm*{\sum_{t=1}^{n}\xi_tX_t-\En\brk*{\xi{}X}}_{\injt_4}
\leq{} \nrm*{\sum_{t=1}^{n}\xi_tX_t'-\En\brk*{\xi{}X'}}_{\injt_4}.
\]
We apply \pref{lem:talagrand} to the truncated complexity, which
implies that with probability at least $1-(\delta+\delta_0)$,
\begin{align*}
&\nrm*{\sum_{t=1}^{n}\xi_tX'_t-\En\brk*{\xi{}X'}}_{\injt_4}\\
&\leq{} 4\En_{X_{1:n}}\En_{\eps}\nrm*{\sum_{t=1}^{n}\eps_t(\tri*{\ts,X_t}-Y_t)X'_t}_{\injt_4}
+ O\prn*{\sqrt{
  \sigma^{2}n\log(1/\delta)}
+ R(d^{3/2}+\tau)\cdot\log(1/\delta)},
\end{align*}
where
$\sigma^{2}=\sup_{\nrm*{T}_{\nuct_4}\leq{}1}\En\brk*{\tri*{T,X'}^{2}(\tri*{\ts,X_t}-Y_t)^{2}}\leq{}\sup_{\nrm*{T}_{\nuct_4}\leq{}1}\nrm*{T}_{\frobnorm}^{2}\cdot{}4R^{2}\leq{}4R^{2}$. We
now apply Lipschitz contraction to bound the in-expectation Rademacher
complexity as
\begin{align*}
 4\En_{X_{1:n}}\En_{\eps}\nrm*{\sum_{t=1}^{n}\eps_t(\tri*{\ts,X_t}-Y_t)X'_t}_{\injt_4}
\leq{}  8R\En_{X'_{1:n}}\En_{\eps}\nrm*{\sum_{t=1}^{n}\eps_tX'_t}_{\injt_4}.
\end{align*}
Integrating out the tail, we can bound the error due to truncation as
\[
\En_{X_{1:n}}\En_{\eps}\nrm*{\sum_{t=1}^{n}\eps_tX'_t}_{\injt_4}
\leq{} \En_{X_{1:n}}\En_{\eps}\nrm*{\sum_{t=1}^{n}\eps_tX_t}_{\injt_4}
+ n\cdot{}\int_{d^{3/2}+\tau}^{\infty}\Pr(\nrm*{X}_{F}>t)dt
\]
We have
\[
n\cdot{}\int_{d^{3/2}+\tau}^{\infty}\Pr(\nrm*{X}_{F}>t)dt
=n\cdot{}\int_{d^{3/2}+\tau}^{\infty}e^{-\frac{t^{2}}{2}}dt
\leq{}n\cdot\exp\prn*{-C\prn*{d^{3}+\tau^{2}}},
\]
for some absolute constant $C$. We pick
$\tau=(\sqrt{(d^{3/2}n+\log(1/\delta))/C})$, so that the quantity
above is $o(1)$ and $\delta_0\leq{}\delta$.
Lastly, since $X_{1:n}$ are
gaussian, \pref{lem:sos_gaussian} implies
\[
\En_{X_{1:n}}\En_{\eps}\nrm*{\sum_{t=1}^{n}\eps_tX_t}_{\injt_4}
=\sqrt{n}\cdot{}\En_{X}\nrm*{X}_{\injt_4}
\leq{}O(d^{3/4}\log^{1/4}d\sqrt{n}).
\]
When $n\geq{}d^{3/2}$, we have $d^{3/2}\leq{}d^{3/4}\sqrt{n}$, and so we
conclude that with probability at least $1-\delta$,
\[
\nrm*{\sum_{t=1}^{n}\xi_tX_t-\En\brk*{\xi{}X}}_{\injt_4}\leq{}O(Rd^{3/4}\log^{3/2}(d/\delta)\sqrt{n}),
\]
so $M=O(Rd^{3/4}\log^{3/2}(d/\delta))$.

\end{proof}

%%% Local Variables:
%%% mode: latex
%%% TeX-master: "paper"
%%% End:

\section{Proofs from \pref{sec:lower_bounds}}
\label{app:lower_bounds}
% !TEX root = paper.tex

\begin{proof}[\pfref{thm:lower_bound}]
Let $\veps>0$ be fixed and let $m=d^{3/2-\veps/2}$. 

We will generate an instance of tensor completion from the 3-XOR instance by treating the indices $(i,j,k)$ in each clause as an observed entry. Precisely, for each clause, we set the corresponding $X = e_i\tens{}e_j\tens{}e_k$, and the corresponding $Y = z_{ijk}$. 
The induced risk for each tensor $T$ in this setting is simply $L_{\cD_Z}(T)\ldef{}\frac{1}{d^{3}}\sum_{i,j,k}(T_{i,j,k}-Z_{i,j,k})^{2}$, 
where the tensor $Z\in(\pmo^{d})^{\tens{}3}$ is defined as follows: in the random case, the entries of $Z$ are selected uniformly at random from $\pmo$; in the planted case with planted assignment $a\in\pmo^d$, we start off with the tensor $a^{\tens{}3}$, then get $Z$ by flipping each coordinate independently with probability $\eta$. 

If we set $n=\Omega(m)$, then we are guaranteed to receive $\wt{\Omega}(m)$ unique clauses under with-replacement sampling, since there are at most logarithmically many repeats with inverse polynomial probability for $m=o(d^{3/2})$. With this observation, it is clear that any algorithm that takes as input the examples $S$ and outputs ``Planted'' with probability at least $1-o(1)$ when $Z$ comes from the planted distribution and ``Random'' with probability at least $1-o(1)$ when $Z$ is from the random distribution is a successful distinguisher for the planted \xor problem with $\wt{\Theta}(m)$ clauses. Going forward we assume $n=\Omega(d)$, since this is the interesting regime for distinguishing and refutation.

Now, suppose we have an algorithm that enjoys the excess risk bound \pref{eq:lower_bound_risk} for any $n$, and let $\wh{T}_{S}$ denote its output given input dataset $S$. Since $Z\in(\pmo^{d})^{\tens{}3}$ we can take $\nrm*{\wh{T}_{S}}_{\infty}\leq{}1$ without loss of generality. We will turn $\wh{T}_S$ into a successful distinguishing algorithm for the planted \xor problem as follows: Define $\gamma=1-4\eta$, and note the assumption that $\eta<1/4$ implies that $\gamma>0$. Split the sample set $S$ into halves $S_1$ and $S_2$, and let $\wh{L}_{S_1}$ and $\wh{L}_{S_2}$ denote the empirical risk on the respective halves. Let $\wh{T}_{S_1}$ be the output of the assumed algorithm on $S_1$. If $\wh{L}_{S_2}(\wh{T}_{S_1})\geq{}1-\gamma^{4}/2$ return ``Random'', else return ``Planted''.

We will show that the algorithm succeeds in both the planted and random case by analyzing the value of $\wh{L}_{S_2}(\wh{T}_{S_1})$ in each case.

\emph{Random case.} Let $S'_{2}$ be the result of removing all entries that appear in $S_1$ from $S_2$. The number of repeats is at most $O(n^2/d^{3} + \log(1/\delta))$ with probability at least $1-\delta$. It follows that with probability at least, say, $1-O(d^{-1})$,
\[
\frac{\abs*{S'_2}}{\abs*{S_2}}\geq{}1-o(1),\quad\text{and}\quad\wh{L}_{S_{2}}(\wh{T}_{S_1})\geq{}\wh{L}_{S'_{2}}(\wh{T}_{S_1}) - o(1).
\]
Observe that for every remaining example $X=e_i\tens{}e_j\tens{}e_k$ in $S'_2$, the value of $\wh{T}_{S_1}$ is statistically independent of the value of $Z_{i,j,k}$. Abbreviating $\wh{T}_{S_1}$ to $\wh{T}$, this means that we have
\[
\En_Z\brk*{\wh{L}_{S'_{2}}(\wh{T}_{S_1})}=\frac{1}{\abs*{S'_2}}\sum_{X=e_i\tens{}e_j\tens{}e_k\in{}S'_2}\En_{Z_{i,j,k}\in\pmo}\prn*{\wh{T}_{i,j,k}-Z_{i,j,k}}^{2}
= \frac{1}{\abs*{S'_2}}\sum_{X=e_i\tens{}e_j\tens{}e_k\in{}S'_2}(\wh{T}_{i,j,k})^{2} + 1 \geq{} 1.
\]
Condition on $S'_2$. Since $\wh{T}_{S_1}$ and $Z$ have bounded entries, it follows from Hoeffding's inequality that with probability at least $1-\delta$ over the choice of $Z$,
\[
\wh{L}_{S'_{2}}(\wh{T}_{S_1}) \geq{}\En_{Z}\brk*{\wh{L}_{S'_{2}}(\wh{T}_{S_1})} - c\sqrt{\frac{\log(1/\delta)}{\abs*{S'_{2}}}}, % - c_2\frac{\log(1/\delta)}{\abs*{S'_{2}}},
\]
for absolute constant $c>0$. Since $\abs*{S_2'}=\wt{\Omega}(m)$, we can use the AM-GM inequality to conclude with probability at least $1-O(d^{-1})$,
\[
\wh{L}_{S_{2}}(\wh{T}_{S_1})\geq{}\wh{L}_{S'_{2}}(\wh{T}_{S_1}) - o(1).  % \geq{}1-\gamma^4/100 - o(1).
\]
It follows by union bound that
\[
\wh{L}_{S_{2}}(\wh{T}_{S_1})\geq{}1 - o(1) \geq{}1-\gamma^4/100 - o(1).
\]
This proves that our strategy will indeed return ``Random'' for random instances.

\emph{Planted case.}
For any $Z$, the assumed excess risk bound implies that for any rank-$1$ tensor $T^{\star}$ with bounded entries, we have
\[
L_{\cD_{Z}}(\wh{T}_{S_1}) - L_{\cD_{Z}}(T^{\star})  \leq{} o(1),
\]
when $n=\omega(d^{3/2-\veps})$. We choose $T^{\star}=a^{\tens{}3}$, where $a\in\pmo^{d}$ is the planted assignment. Taking expectation over the flips in $Z$, we have
\[
\En_{Z}\brk*{L_{\cD_{Z}}(a^{\tens3})} = \En_{Z}\brk*{\frac{4}{d^{3}}\sum_{i,j,k}\ind\crl*{\text{$Z_{ijk}$ flipped}}}=4\eta.
\]
Applying Bernstein's inequality, we have that with probability at least $1-O(d^{-1})$ over the choice of $Z$,
\[
L_{\cD_{Z}}(a^{\tens3})\leq{} 4(1+\gamma)\eta + o(1)= 1-\gamma^{2} + o(1).
\]
Finally, we use Bernstein once more to show that the empirical loss for $S_2$ converges to the population loss, leveraging that $S_{2}$ is an independent hold-out set for $\wh{T}_{S_1}$. We have that for any choice of $S_1$, with probability at least $1-\delta$ over the draw of $S_2$,
\[
\wh{L}_{S_{2}}(\wh{T}_{S_1})
\leq{} L_{\cD_Z}(\wh{T}_{S_1}) + c_1\sqrt{\frac{L_{\cD_Z}(\wh{T}_{S_1})\log(1/\delta)}{n}} + c_2\frac{\log(1/\delta)}{n},
\]
where $c_1$ and $c_2$ are absolute constants. Applying AM-GM, this implies that $\wh{L}_{S_{2}}(\wh{T}_{S_1})\leq{}(1+\gamma^{2})L_{\cD_Z}(\wh{T}_{S_1}) + o(1)$ with probability at least $1-O(d^{-1})$, and so by combining this with the excess risk bound and the bound on $L_{\cD_Z}(a^{\tens{}3})$, we have
\[
\wh{L}_{S_{2}}(\wh{T}_{S_1}) \leq{} 1-\gamma^{4} + o(1).
\]

\end{proof}

%%% Local Variables:
%%% mode: latex
%%% TeX-master: "paper"
%%% End:

\end{document}